\documentclass[a4paper,10pt]{article}
\usepackage[utf8]{inputenc}

\usepackage{times}

\usepackage{graphicx}

\usepackage{todonotes}

\usepackage[scaled=.9]{helvet} 
\usepackage{courier}

\usepackage{nameref}
\usepackage{amssymb}
\usepackage{graphicx}
\usepackage{textcomp}
\usepackage{tikz} 

\usepackage[ruled,vlined]{algorithm2e}

\usepackage[T1,hyphens]{url}

\usepackage{color}

\newtheorem{lemma}{Lemma}
\newtheorem{theorem}{Theorem}

\newtheorem{property}{Property} 
\newtheorem{corollary}{Corollary}
\newtheorem{definition}{Definition}
\newtheorem{example}{Example}

\newenvironment{proof} {\noindent\emph{Proof:}}{$\left.\right.$\hfill$\square$}

\SetKwInput{KwAccess}{Access}

\newcommand{\rew}{\texttt{rew} }

\newcommand{\cover}{\texttt{cover} }

\newcommand{\fun}[1]{\ensuremath{\mbox{\sl #1}}}

\newcommand{\terms}[1]{\ensuremath{\fun{terms}(#1)}}
\newcommand{\vars}[1]{\ensuremath{\fun{vars}(#1)}}
\newcommand{\consts}[1]{\ensuremath{\fun{consts}(#1)}}

\newcommand{\bod}[1]{\ensuremath{\fun{body}(#1)}}
\newcommand{\head}[1]{\ensuremath{\fun{head}(#1)}}

\newcommand{\fr}[1]{\ensuremath{\fun{fr}(#1)}}

\newcommand{\sep}[1]{\ensuremath{\fun{sep}(#1)}}

\usepackage{authblk}

\title{Sound, Complete and Minimal UCQ-Rewriting for Existential Rules}
\author[1]{M\'elanie K\"onig}
\author[1]{Michel Lecl\`ere}
\author[1]{Marie-Laure Mugnier }
\author[2]{Micha\"el Thomazo\footnote{This work was done when M. Thomazo was a PhD student at University Montpellier 2.}}
\affil[1] {University Montpellier 2, France}
\affil[2] {TU Dresden, Germany}

\date{}

\begin{document}

\maketitle

\begin{abstract}
We address the issue of Ontology-Based Data Access, with ontologies represented in the framework of existential rules, also known as Datalog+/-. 
A well-known approach involves rewriting the query using ontological knowledge. 
We focus here on the basic rewriting technique which consists of rewriting the initial query into a union of conjunctive queries. 
First, we study a generic breadth-first rewriting algorithm, which takes as input any rewriting operator, and define properties of rewriting operators that ensure the correctness of the algorithm. Then, we focus on piece-unifiers, which provide a rewriting operator with the desired properties. Finally, we propose an implementation of this framework and report some experiments. 
 \end{abstract}

\section{Introduction}

We address the issue of Ontology-Based Data Access, which aims at  exploiting knowledge expressed in ontologies while querying data. 
In this paper, ontologies are represented in the framework of existential rules \cite{blms:11,kr:11}, also known as Datalog$\pm$ \cite{cali-gottlob-kifer:08,cali:09}.  
Existential rules allow one to assert the existence of new unknown individuals, which is a key feature in an open-world
perspective, where data are incompletely represented. 
These rules  are of the form \emph{body} $\rightarrow$ \emph{head}, where the body and the head are conjunctions of atoms (without functions) and variables that occur only in  the head are \emph{existentially} quantified. They generalize lightweight description logics (DLs), which form the core of the tractable profiles of OWL2.

The general query answering problem can be expressed as follows: given a knowledge base (KB) $\mathcal K$ composed of a set of facts -or data- and an ontology (a set of existential rules here), and a query $Q$, compute the set of answers to $Q$ in $\mathcal K$. In this paper, we consider Boolean conjunctive queries (Boolean CQs or BCQs). Note however that all our results are easily extended to non-Boolean conjunctive queries as well as to unions of conjunctive queries.  The fundamental problem, called BCQ entailment hereafter, can be recast as follows: given a KB $\mathcal K$ composed of facts and existential rules, and a Boolean conjunctive query $Q$, is $Q$ entailed by  $\mathcal K$?

BCQ entailment is undecidable for general existential rules. There has been an
intense research effort aimed at finding decidable subsets of rules that provide good
tradeoffs between expressivity and complexity of query answering (see e.g. \cite{rr-11-m} for a synthesis). 
With respect to lightweight DLs, these decidable rule fragments are more powerful and
flexible.  In particular, they have unrestricted predicate arity, while DLs consider unary and
binary predicates only, which allows one for a natural coupling with database schemas, in
which relations may have any arity; moreover, adding pieces of information, for instance to
take contextual knowledge into account, is made easier by the unrestricted predicate arity,
since they can be added as new predicate arguments. 

There are two main approaches to solve BCQ entailment, which are linked to the classical paradigms for processing rules, namely forward and backward chaining, schematized in Figure 1. Both can be seen as ways of reducing the problem to a classical database query answering problem by eliminating the rules. The first approach consists in applying the rules to the data, thus materializing entailed facts into the data. Then, $Q$ is entailed by $\mathcal K$ if and only if it can be mapped to this materialized database. The second approach consists in using the rules to rewrite the query into a first-order query (typically a union of conjunctive queries \cite{dl-lite:07,eswc-09-phm,icde-11-gop,dl-12-vss,kr-12-rc}) or a non-recursive Datalog program \cite{kr-10-ra, kr-12-gs}. Then, $Q$ is entailed by $\mathcal K$ if and only if the rewritten query is entailed by the initial database. 
Materialization has the advantage of enabling efficient query answering but may be not appropriate for size, data access rights or data maintenance reasons. Query rewriting has the advantage of avoiding changes in the data, however its drawback is that the rewritten query may be large, even exponential in the size of initial query, hence less efficiently processed, at least with current database techniques. Finally, techniques combining both approaches have been developed, in particular so-called combined approach \cite{lutz:09,ijcai-11-kltwz}. 

\begin{figure}
\begin{center}
\includegraphics[width=12cm]{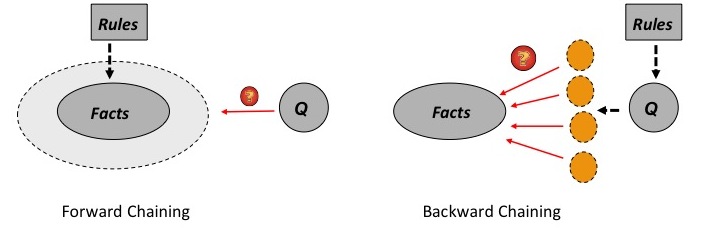} 
\caption{Forward / Backward Chaining}
\end{center}
\label{fig:FC-BC}
\end{figure}

In this paper, we focus on rewriting techniques, and more specifically on rewriting the initial conjunctive query $Q$ into a union of conjunctive queries, that we will see as a \emph{set} of conjunctive queries, called \emph{rewritings} of $Q$. While previously cited work focuses on specific rule sublanguages, we consider general existential rules.  The goal is to compute a set of rewritings both \emph{sound} (if one of its elements maps to the initial database, then $\mathcal K$ entails $Q$) and \emph{complete} (if $\mathcal K$ entails $Q$ then there is an element that maps to the initial database). Minimality may also be a desirable property. In particular, let us consider the generalization relation (a preorder) induced on Boolean conjunctive queries by homomorphism: we say that  $Q_1$ is more general than $Q_2$ if there is a homomorphism from $Q_1$ to $Q_2$; it is well-known that the existence of such a homomorphism is equivalent to the following property: for any set of facts $F$, if the answer to $Q_2$  in $F$ is positive so is the answer to $Q_1$. 
We point out that any sound and complete set of rewritings of a query $Q$ remains sound and complete when it is restricted to its most general elements. 
Since BCQ entailment is undecidable, there  is no guarantee that such a \emph{finite} set exists for a given query and general existential rules.  A set of existential rules ensuring that  a finite sound and complete set of most general rewritings exists for \emph{any} query is called a \emph{finite unification set}  (\emph{fus}) \cite{blms:11}. The \emph{fus} property is not recognizable \cite{blms:11}, but several easily recognizable \emph{fus}  classes have been exhibited in the literature: atomic-body rules \cite{blms:09}, also known as linear TGDs \cite{cali:09}, multi-linear \cite{jws-12-cgl},(join-)sticky rules \cite{cgp:10}, weakly-recursive rules \cite{d20-12-cr} and sets of rules with an acyclic graph of rule dependencies \cite{blms:09}.

\paragraph{Paper contributions.}

We start from a generic algorithm which, given a BCQ and a set of existential rules, computes a rewriting set. This task can be recast in terms of exploring a potentially infinite space of queries, composed of the initial conjunctive query and its (sound) rewritings, structured by the generalization preorder.   The algorithm explores this space in a breadth-first way, with the aim of computing a complete set of rewritings. It maintains a set of rewritings $\mathcal Q$ and iteratively performs the following tasks: (1) generate all the one-step rewritings from unexplored queries in $\mathcal Q$; (2) add these rewritings to $\mathcal Q$ and update  $\mathcal Q$ in order to keep only incomparable most general elements. We call \emph{rewriting operator} the function that, given a query and a set of rules, returns the one-step rewritings of this query. Note that it may be the case that the set of sound rewritings of the query is infinite while the set of its most general sound rewritings is finite. It follows that a simple breadth-first exploration 
of the rewriting space is not sufficient to ensure finiteness of the process, even for \emph{fus} rules; one also has to maintain a set of the most general rewritings.
This algorithm is generic in the sense that it is not restricted to a particular kind of existential rules nor to a specific rewriting operator. 

This algorithmic scheme established, we then asked ourselves the following questions:
\begin{enumerate} 
 \item Assuming that the algorithm outputs a finite sound and complete rewriting set of pairwise incomparable queries, is this set of minimal cardinality, 
 in the sense that no sound and complete set of rewritings produced by any other algorithm can be strictly smaller?
 \item At each step of the algorithm, some queries are discarded, because they are more specific than other rewritings, even if they have not been explored yet. The question is whether this dynamic pruning of the search space keeps the completeness of the output. More generally, which properties have to be fulfilled by the operator to ensure the correctness of the algorithm and its termination for \emph{fus} rules?
 \item Finally, design a rewriting operator that fulfills the desired properties and leads to the effective computation of the rewriting set.   
 \end{enumerate}

With respect to the first question, we show that all sound and complete rewriting sets restricted to their most general elements have the \emph{same} cardinality, which is minimal with respect to the completeness property. If we moreover delete redundant atoms from the obtained CQs (which can be performed by a linear number of homomorphism tests for each query), we obtain a unique minimal sound and complete set of CQs of minimal size; unicity is of course up to a bijective variable renaming. 

To answer the second question, we define several properties that a rewriting operator has to satisfy and show that these properties actually ensure the correctness of the algorithm and its halting for \emph{fus} rules. In particular, we point out that the fact that a query may be removed from the rewriting set before being explored may prevent the completeness of the output, even if the rewriting operator is theoretically able to generate a complete output. The \emph{prunability} of the rewriting operator ensures that this dynamic pruning can be safely performed. Briefly, this property holds if, for all queries $Q_1$ and $Q_2$, when $Q_1$ is more general than $Q_2$ then any one-step rewriting of $Q_2$ is less general than $Q_1$ itself or one of the one-step rewritings of  $Q_1$; intuitively, this allows to discard the rewriting $Q_2$ even when its one-step rewritings have not been generated yet. 
Note that this kind of properties ties in with an issue raised in \cite{isg:12} about the gap between theoretical completeness of some methods and the effective completeness of their implementation, this gap being mainly due to algorithmic optimizations (here the dynamic pruning).  

Concerning the third question, we proceed in several steps.  First, we rely on a specific unifier, called a \emph{piece-unifier}, that was designed for backward chaining with conceptual graph rules (whose logical translation is exactly existential rules  \cite{salvat-mugnier:96}). As in classical backward chaining, the rewriting process relies on a unification operation between the current query and a rule head. However, existential variables in rule heads induce a structure that has to be considered to keep soundness.  Thus, instead of unifying a single atom of the query at once, our unifier processes a subset of atoms from the query. We call \emph{piece} a minimal subset of atoms from the query that have to be erased together, hence the name piece-unifier.  We present below a very simple example of piece unification (in particular, the head of the existential rule is restricted to a single atom).

\begin{example}\label{ex:piece1}  Let $R =  \forall x~(q(x) \rightarrow \exists y~p(x,y))$ and the BCQ $Q = \exists u \exists v \exists w (p(u,v) \wedge p(w,v) \wedge r(u,w))$.  Assume we want to unify the atom $p(u,v)$ from $Q$ with $p(x,y)$, for instance by a substitution $\{(u,x),(v,y)\}$. Since $v$ is unified with the existential variable $y$,  all other atoms from $Q$ containing $v$ must also be considered: indeed, simply rewriting $Q$ into $Q_1 = q(x) \wedge p(w,y)  \wedge r(x,w)$ would be unsound: intuitively, the fact that the atoms $p(u,v)$ and $p(w,v)$ in $Q$ share a variable would be lost in atoms $q(x)$ and $p(w,y)$; for instance $F=q(a) \wedge p(b,c)  \wedge r(a,b)$ would answer $Q_1$ despite  $Q$ is not entailed by $F$ and $R$. Thus, $p(u,v)$ and $p(w,v)$ have to be both unified with the head of $R$, for instance by means of the following substitution: $\mu = \{(u,x), (v,y),(w,x)\}$. $\{p(
u,v),p(w,v)\}$ is called a piece. The corresponding rewriting of $Q$ 
is $q(x) \wedge r(x,x)$. 
\end{example}

Piece-unifiers lead to a logically sound and complete rewriting method. As far as we know, it is the only method accepting any kind of existential rules, while staying in this fragment, i.e., without Skolemization of rule heads to replace existential variables with Skolem functions. 

We show that the piece-based rewriting operator fulfills the desired properties ensuring the correctness of the generic algorithm and its termination in the case of \emph{fus} rules. 
The next question was how to optimize the rewriting step.  Indeed, the problem of deciding whether there is a piece-unifier between a query and a rule head is NP-complete 
  and the number of piece-unifiers can be exponential in the size of the query.   To cope with these sources of complexity, we consider so-called \emph{single-piece} unifiers, which unify a single-piece of the query at once (like $\mu$ in Example \ref{ex:piece1}).  We also focus on rules with a head restricted to an atom. This is not a restriction in terms of expressivity, since any rule can be decomposed into an equivalent set of atomic-head rules by simply introducing a new predicate for each rule (e.g. \cite{cali-gottlob-kifer:08}, \cite{blms:09}). The interesting point is that each atom in $Q$ belongs to at most one piece with respect to $R$ when $R$ has an atomic head (which is false for general existential rules). In the case of rules with atomic head, the number of (most general) single-piece unifiers of a query $Q$ with the head of a rule $R$ is bounded by the size of the query. We show that the single-piece based rewriting operator is able to generate a sound and complete set of rewritings. However, as pointed out in several examples, it is not prunable. Hence, single-piece unifiers have to be combined to recover prunability. We thus define the \emph{aggregation} of single-piece unifiers and show that the corresponding rewriting operator fulfills all desired properties and generates less queries than the piece-based rewriting operator. Detailed algorithms are given and first experiments are reported.  

\paragraph{Paper organization.}

Section \ref{sec:prelim} recalls some basic notions about the existential rule framework. Section \ref{sec:prop} defines sound, complete and minimal sets of rewritings. In Section \ref{sec:generic} the generic breadth-first algorithm is introduced and  general properties of rewriting operators are studied. Section \ref{sec:pu} presents the piece-based rewriting operator.
 In Section \ref{sec:algo}, we focus on exploiting single-piece unifiers and introduce the rewriting operator based on their aggregation. Finally, Section \ref{sec:expe} is devoted to implementation and experiments, as well as to further work.
  
 This is an extended version of papers by the same authors published at RR 2012 and RR 2013 (International Conference on Web Reasoning and Rule Systems).

\section{Preliminaries}
\label{sec:prelim}

An \emph{atom} is of the form $p(t_1, \ldots, t_k)$ where $p$ is a predicate with arity
$k$, and the $t_i$ are terms, i.e., variables or constants.
Given an atom or a set of atoms $A$, $\fun{vars}(A)$,
$\fun{consts}(A)$ and $\fun{terms}(A)$ denote its set of variables, of constants and of
terms, respectively. In all the examples in this paper, the terms are variables (denoted by
$x$, $y$, $z$, etc.).  $\models$ denotes the classical logical
consequence. Two formulas $f_1$ and $f_2$ are said to be equivalent if $f_1 \models f_2$ and
$f_2 \models f_1$. 

A \emph{fact} is an existentially closed conjunction of atoms.\footnote{We generalize
the classical notion of a fact in order to take existential variables into account.}  A
\emph{conjunctive query} (CQ) is an existentially quantified conjunction of atoms. When it
is a closed formula, it is called a \emph{Boolean} CQ (BCQ). Hence facts and BCQs have
the same logical form. In the following, we will see them as sets of atoms. 
 Given sets of atoms  $A$ and $B$, a
\emph{homomorphism} $h$ from $A$ to $B$ is a
substitution
of $\vars{A}$ by $\terms{B}$ s.t. $h(A) \subseteq B$. We say that $A$ is \emph{mapped} to
$B$ by $h$. If there is a homomorphism from $A$ to $B$, we say that $A$ is \emph{more
general} than $B$, which is denoted $A \geq B$.  

Given a fact $F$ and a BCQ $Q$, the answer to $Q$ in $F$ is \emph{positive} if
 $F \models Q$.  It is well-known that $F \models Q$ if and only if there is a homomorphism
from $Q$ to $F$.  If $Q$ is a non-Boolean CQ, let $x_1 \ldots x_q$ be the free variables in $Q$. Then, a tuple of constants $(a_1 \ldots a_q)$ is an answer to $Q$ in $F$ if there is a homomorphism from $Q$ to $F$ that maps $x_i$ to $a_i$ for each $i$. In other words, $(a_1 \ldots a_q)$ is an answer to $Q$ in $F$  if and only if  the answer to the BCQ obtained from $Q$ by substituting each  $x_i$ with $a_i$ is positive. 

In this paper, we consider only Boolean queries for simplicity reasons. This is not a restriction, since our mechanisms can actually process a CQ with free variables $x_1 \ldots x_q$  by translating it into a BCQ with an added atom $ans(x_1\ldots x_q)$, where $ans$ is a special predicate not occurring in the knowledge base. Since $ans$ can never be erased by a rewriting step, the $x_i$ can only be substituted and will not ``disappear''. We can thus compute the set of rewritings of a CQ as a Boolean CQ with a special $ans$ atom, then transform the rewritings into  non-Boolean CQs by removing the $ans$ atom and consider its arguments as free variables. 
Note that our the generic algorithm can accept as input a union of conjunctive queries as well, since it works exactly in the same way if it takes as input a set of CQs instead of a single CQ. 

\begin{definition}[Existential rule] An \emph{existential rule} (or simply a \emph{rule}) is a formula 
$R = \forall \vec{x} \forall
\vec{y}(B[\vec{x},\vec{y}] \rightarrow \exists \vec{z} H[\vec{y},\vec{z}])$ where $\vec{x},\vec{y}$ and $\vec{z}$ are tuple of variables, $B = \bod{R}$  and $H = \head{R}$ are conjunctions of atoms, resp. called the \emph{body}
and the \emph{head} of $R$. The \emph{frontier} of $R$, noted $\fr{R}$, is the set  $\vars{B} \cap \vars{H}  = \vec{y}$. 
The set of \emph{existential variables} in $R$ is the set $ \vars{H} \setminus \fr{R} = \vec{z}$. 
\end{definition}

In the following, we will omit quantifiers in rules as there is no ambiguity.

A \emph{knowledge base} (KB) $\mathcal K = (F, \mathcal R)$ is composed of a 
fact $F$ and a finite set of existential rules $\mathcal R$.
 The \emph{BCQ entailment problem} takes as input a KB $\mathcal K = (F,
\mathcal R)$ and a BCQ $Q$, and asks if $F, \mathcal R \models Q$  holds.

\section{Desirable Properties of Rewriting Sets}
\label{sec:prop}

 Given a query $Q$ and a set of existential rules $\mathcal R$, rewriting techniques compute a set of queries $\mathcal Q$, which we call a  \emph{rewriting set} hereafter. It is generally desired that such a set satisfies at least three properties:  \emph{soundness},  \emph{completeness} and \emph{minimality}.

\begin{definition}[Sound and Complete set]
\label{def-sound-complete-rewriting}
Let $\mathcal R$ be a set of existential rules and $Q$ be a BCQ. Let $\mathcal Q$
be a set of BCQs. $\mathcal Q$ is said to be \emph{sound} w.r.t. $Q$ and $\mathcal R$ if
for all facts $F$, for all $Q'\in \mathcal Q$, if $Q'$ can be mapped to $F$ then ($\mathcal R, F
\models Q$). Reciprocally, $\mathcal Q$ is said to be \emph{complete} w.r.t. $Q$ and
$\mathcal R$ if for all fact $F$, if ($\mathcal R, F \models Q$) then there is $Q' \in
\mathcal Q$ s.t. $Q'$ can be mapped to $F$.
\end{definition}

We mentioned in the introduction 
that only the \emph{most general elements} of a rewriting set need to be considered. Indeed, let $Q_1$ and $Q_2$ be two elements of a rewriting set such that $Q_1 \geq Q_2$ and let $F$ be any fact: if $Q_1$ maps to $F$, then $Q_2$ is useless; if $Q_1$ does not map to $F$, neither does $Q_2$; thus removing $Q_2$ will not undermine completeness (and it will not undermine soundness either). The output of a rewriting algorithm should thus be a minimal set of incomparable queries that ``covers'' the set of all the sound rewritings  of the initial query. 

\begin{definition}[Covering relation]
Let $\mathcal Q_1$ and  $\mathcal Q_2$ be two sets of BCQs. $\mathcal Q_1$ \emph{covers} $\mathcal Q_2$, which is denoted $\mathcal Q_1\geq \mathcal Q_2$, if for all $Q_2\in \mathcal Q_2$ there is $Q_1\in \mathcal Q_1$ with $Q_1 \geq Q_2$.
\end{definition}

\begin{definition}[Minimal set of BCQs, Cover]
Let $\mathcal Q$ be a set of BCQs. $\mathcal Q$ is said to be \emph{minimal} if there is no $Q \in \mathcal Q$ such that $(\mathcal Q \setminus\{Q\})\ \geq \ \mathcal Q$.  A \emph{cover} of $\mathcal Q$ is a minimal set $\mathcal Q^c \subseteq \mathcal Q$ such that $\mathcal Q^c\geq \mathcal Q$. 
\end{definition}

Since a cover is a minimal set, its elements are pairwise incomparable.
\begin{figure}
\begin{center}
\includegraphics[height=5cm]{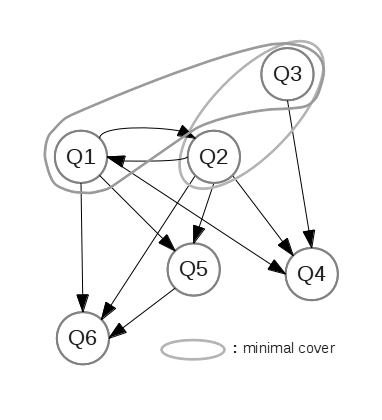}
\caption{Cover (Example \ref{ex:cover})\label{fig:cover}}
\end{center}
\end{figure}

\begin{example} 
\label{ex:cover}See also Figure \ref{fig:cover}. Let $\mathcal Q = \{Q_1, \ldots, Q_6\}$ and the following preorder over  $\mathcal Q$:
$Q_1 \geq Q_2, Q_4, Q_5, Q_6$ ;  $Q_2 \geq Q_1, Q_4, Q_5, Q_6$ ;  $Q_3 \geq Q_4$ ; $Q_5 \geq Q_6$ (note that $Q_1$ and $Q_2$ are equivalent).
There are two covers of $\mathcal Q$, namely $\{Q_1, Q_3\}$ and $\{Q_2, Q_3\}$.
\end{example}
A set of (sound) rewritings may have a finite cover even when it is infinite, as illustrated by Example \ref{ex-most-gen-infinite}.

\begin{example}
\label{ex-most-gen-infinite}
Let $Q = t(u), R_1 = t(x) \wedge p(x,y) \rightarrow r(y), R_2 = r(x) \wedge p(x,y) \rightarrow t(y)$. $R_1$ and $R_2$ have a head restricted to a single atom and no existential variable, hence the classical most general unifier can be used, which unifies the first atom in the query with the atom of a rule head. 
The set of rewritings of $Q$ with $\{R_1,R_2\}$ is infinite. The first generated queries are the following (note that rule variables are renamed when needed):\\
$Q_0 = t(u)$\\
$Q_1= r(x) \wedge p(x,y)$ // from $Q_0$ and $R_2$ with $\{(u,y)\}$\\
$Q_2 = t(x_0) \wedge p(x_0, y_0) \wedge p(y_0, y)$ // from $Q_1$ and $R_1$ with $\{(x,y_0)\}$ \\
$Q_3 = r(x_1) \wedge p(x_1, y_1) \wedge p(y_1, y_0) \wedge p(y_0, y)$ // from $Q_2$ and $R_2$ with $\{(x_0,y_1)\}$ \\
$Q_4 = t(x_2) \wedge p(x_2, y_2) \wedge p(y_2,y_1) \wedge p(y_1, y_0) \wedge p(y_0, y)$ // from $Q_3$ and $R_1$ \\
\emph{and so on} $\dots$\\
However, the set of the most general rewritings is $\{Q_0, Q_1\}$ since any other query than can be obtained is more specific than $Q_0$ or $Q_1$. 
\end{example}

It can be easily checked that all covers of a given set have the same cardinality. We now prove that this property can be extended to the covers of all sound and complete finite rewriting sets of $Q$, no matter of the rewriting technique used to compute these sets. 

\begin{theorem}\label{prop-mgr}
Let $\mathcal R$ be a set of rules and $Q$ be a BCQ. Any finite cover of a sound and complete rewriting set of $Q$ with $\mathcal R$ is of minimal cardinality (among all sound and complete rewriting sets of $Q$). 
\end{theorem}

\begin{proof}
Let $\mathcal Q_1$ and $\mathcal Q_2$ be two arbitrary sound and complete rewriting sets of $Q$ with $\mathcal R$, and $\mathcal Q_1^c$ and $\mathcal Q_2^c$ be one of their respective finite covers. $\mathcal Q_1^c$ and $\mathcal Q_2^c$ are also sound and complete, and are of smaller cardinality. We show that they have the same cardinality. Let $Q_1 \in \mathcal Q_1^c$. There exists $Q_2 \in \mathcal Q_2^c$ such that $Q_2 \geq Q_1$. If not, $Q$ would be entailed by $F= Q_1$ and $\mathcal R$ since $\mathcal Q_1^c$ is a sound rewriting set of $Q$ (and $Q_1$ maps to itself), but no elements of $\mathcal Q_2^c$ would map to $F$: thus, $Q_2^c$  would not be complete. Similarly, there exists $Q'_1 \in \mathcal Q_1^c$
 such that $Q'_1 \geq Q_2$. Then $Q'_1 \geq Q_1$, which implies that $Q'_1 = Q_1$ by assumption on $Q_1^c$. For all $Q_1 \in \mathcal Q_1^c$, there exists $Q_2 \in \mathcal Q_2^c$ such that $Q_2 \geq Q_1$ and $Q_1 \geq Q_2$. Such a $Q_2$ is unique: indeed, two such elements would be comparable for $\geq$, which is not possible by construction of $\mathcal Q_2^c$. The function associating $Q_2$ with $Q_1$ is thus a bijection from $\mathcal Q_1^c$ to $\mathcal Q_2^c$, which shows that these two sets have the same cardinality.
\end{proof}

Furthermore, the proof of the preceding theorem shows that, given any two sound and complete rewriting sets of $Q$, there is a bijection from any cover of the first set to any cover of the second set such that two elements in relation by the bijection are equivalent. However, these elements are not necessarily isomorphic  (i.e., equal up to a variable renaming) because they may contain redundancies. Consider the preorder induced by homomorphism on the set of all BCQs definable on some vocabulary. It is well-known that this preorder is such that any of its equivalence classes possesses a unique element of minimal size (up to isomorphism), called its \emph{core} (notion introduced for graphs, but easily transferable to queries). Every query can be transformed into its equivalent core by removing redundant atoms. We recall that a set of existential rules ensuring that a finite sound and complete set of most general rewritings exists for any query is called a finite unification set (\emph{fus}).\footnote{The 
finite unification set notion was first introduced in \cite{blms:09} and defined with respect to piece-unifiers. However, since piece-unifiers provide a sound and complete rewriting operator, as recalled in Section \ref{sec:pu}, and all the covers of a given set have the same cardinality,  both definitions are equivalent.}

From previous remark and Theorem  \ref{prop-mgr}, we obtain:

\begin{corollary}
Let  $\mathcal R$ be a fus and $Q$ be a BCQ. There is a unique finite sound and complete rewriting set of $Q$ with $\mathcal R$ that has both minimal cardinality and elements of minimal size. 
\end{corollary}

\section{A Generic Breadth-First Algorithm}
\label{sec:generic}

We will now present a generic rewriting algorithm that takes as input a set of existential rules and a query, and as parameter a \emph{rewriting operator}. The studied question is the following: which properties should this operator fulfill in order that the algorithm outputs a sound, complete, finite and minimal set?

\subsection{Algorithm}

\begin{definition}[Rewriting operator]
A \emph{rewriting operator} \rew is a function which takes as input a conjunctive query $Q$ and a set of rules $\mathcal R$ and outputs a set of conjunctive queries denoted by $\rew(Q,\mathcal R)$.
\end{definition}

Since the elements of $\rew(Q,\mathcal R)$ are queries, it is possible to apply further steps of rewriting to them. This naturally leads to the notions of $k$-rewriting and $k$-saturation.

\begin{definition}[$k$-rewriting]
Let $Q$ be a conjunctive query, $\mathcal R$ be a set of rules and $\rew$ be a rewriting operator. A \emph{$1$-rewriting} of $Q$ (w.r.t. $\rew$ and $\mathcal R$) is an element of $\rew(Q,\mathcal R)$. A \emph{$k$-rewriting }of $Q$, for $k > 1$, (w.r.t. $\rew$ and $\mathcal R$) is a $1$-rewriting of a $(k-1)$-rewriting of $Q$.
\end{definition}

The term $k$-saturation is convenient to name the set of queries that can be obtained in at most $k$ rewriting steps.

\begin{definition}[$k$-saturation]
Let $Q$ be a query, $\mathcal R$ be a set of rules and $\rew$ be a rewriting operator. We denote by $\rew_k(Q,\mathcal R)$ the set of $k$-rewritings of $Q$. We call \emph{$k$-saturation}, and denote by $W_k(Q,\mathcal R)$, the set of $i$-rewritings of $Q$ for all $i \leq k$. We denote $W_\infty(Q,\mathcal R) = \bigcup_{k \in \mathbb N} W_k(Q,\mathcal R)$.
\end{definition}

In the following, we extend the notations $\rew$, $\rew_k$ and $W_k$ to a set of queries $\mathcal Q$ instead of a single query $Q$:
$\rew(\mathcal Q,\mathcal R) = \bigcup_{Q\in \mathcal Q} \rew(Q,\mathcal R)$, $\rew_k(\mathcal Q,\mathcal R) = \bigcup_{Q\in \mathcal Q} \rew_k(Q,\mathcal R)$ and $W_k(\mathcal Q,\mathcal R) = \bigcup_{i \leq k} \rew_i(\mathcal Q,\mathcal R)$.

Algorithm \ref{algo-gen-fus} performs a breadth-first exploration of the rewriting space of a given query. At each step, only the most general elements are kept thanks to a covering function, denoted by \cover, that computes a cover of a given set. For termination reasons (see the proof of Property \ref{haltProp}), we require that if both $\mathcal Q_c \cup \{q\}$ and $\mathcal Q_c \cup \{q'\}$ are covers of $\mathcal Q_F \cup \rew(\mathcal Q_E,\mathcal R)$, with $q$ and $q'$ homomorphically equivalent and $\{q\}$ belongs to $\mathcal Q_F$, then \cover does not output $\mathcal Q_c \cup \{q'\}$ -- which intuitively means that queries already explored are preferred to non-explored queries in the choice of a cover. If \rew fulfills some good properties (subsequently specified), then after the $i^{th}$ iteration of the while loop the $i$-saturation of $Q$ (with respect to $\mathcal R$ and $\rew$) is covered by $\mathcal Q_F$, while $\mathcal Q_E$ contains the queries that remain to be explored.

\begin{algorithm}[ht]
\KwData{A set of rules $\mathcal R$, a BCQ $Q$}
\KwAccess{A rewriting operator \rew, a covering function \cover}
\KwResult{A cover of the set of all the rewritings of $Q$
}
$\mathcal Q_F \leftarrow \{Q\}$; \emph{// resulting set}\\
$\mathcal Q_E \leftarrow \{Q\}$; \emph{// queries to be explored} \\
\While{$\mathcal Q_E \not = \emptyset$}
{
$\mathcal Q_C \leftarrow \cover(\mathcal Q_F \cup \rew(\mathcal Q_E,\mathcal R))$;
\emph{// update cover}\\
$\mathcal Q_E \leftarrow \mathcal Q_C \backslash \mathcal Q_F$; \emph{// select
unexplored queries}\\
$\mathcal Q_F \leftarrow \mathcal Q_C$;\\
}
\KwRet{$\mathcal Q_F$}
\caption{{\sc A generic breadth-first rewriting algorithm}} \label{algo-gen-fus}
\end{algorithm}

In the remainder of this section,  we study the conditions that a rewriting operator must meet in order that: (i) the algorithm halts  and outputs a cover of all the rewritings that can be obtained with this rewriting operator, provided that such a finite cover exists; (ii) the output cover is sound and complete.

\subsection{Correctness and Termination of the Algorithm}

We now exhibit a sufficient property on the rewriting operator that ensures that Algorithm \ref{algo-gen-fus} outputs a
cover of $W_\infty(Q,\mathcal R)$.

\begin{definition}[Prunable]\label{def-prunable}
Let $\rew$ be a rewriting operator. $\rew$ is \emph{prunable} if for any set of rules $\mathcal R$ and for all queries $Q_1,Q_2,Q'_2$ such that $Q_1 \geq Q_2$, $Q'_2 \in \rew(Q_2,\mathcal R)$ and $Q_1 \not \geq Q'_2$, there is $Q'_1 \in \rew(Q_1,\mathcal R)$ such that $Q'_1 \geq Q'_2$.
\end{definition}

Intuitively, if an operator is prunable then it is guaranteed that for every $Q_1$ more general than $Q_2$, the one-step rewritings of $Q_2$ are covered by the one-step rewritings of $Q_1$ or by $Q_1$ itself. The following lemma states that this can be generalized to $k$-rewritings for any $k$.

\begin{lemma}
\label{lemma:main}
Let $\rew$ be a prunable rewriting operator, and let $\mathcal Q_1$ and $\mathcal Q_2$ be two sets of queries. If $\mathcal Q_1 \geq \mathcal Q_2$, then $W_\infty(\mathcal Q_1,\mathcal R) \geq W_\infty(\mathcal Q_2,\mathcal R)$.
\end{lemma}

\begin{proof}
   We prove by induction on $i$ that
$W_i(\mathcal Q_1,\mathcal R) \geq \rew_i(\mathcal Q_2,\mathcal R)$.\\
   For $i=0$, $W_0(\mathcal Q_1,\mathcal R) = \mathcal Q_1 \geq \mathcal Q_2 =  \rew_0(\mathcal Q_2,\mathcal R)$.\\
   For $i>0$,  for any $Q_2 \in \rew_{i}(\mathcal Q_2,\mathcal R)$, there is $Q'_2 \in \rew_{i-1}(\mathcal Q_2,\mathcal R)$ such that $Q_2 \in \rew(Q'_2,\mathcal R)$. By induction hypothesis, there is $Q'_1 \in W_{i-1}(\mathcal Q_1,\mathcal R)$ such that $Q'_1 \geq Q'_2$. $\rew$ is prunable, thus
 either $Q'_1 \geq Q_2$ or there is $Q_1 \in \rew(Q'_1, \mathcal{R})$ such that $Q_1 \geq Q_2$. Since $W_{i-1}(\mathcal Q_1,\mathcal R)$ and $\rew(Q'_1, \mathcal{R})$ are both included in $W_i(\mathcal Q_1,\mathcal R)$, we can conclude.
\end{proof}

This lemma would not be sufficient to prove the correctness of Algorithm \ref{algo-gen-fus}, as will be discussed in Section \ref{sec:single-prunability}. We need a stronger version, which checks that a query whose $1$-rewritings are covered needs not to be explored.

\begin{lemma}
\label{lemma:prunable}
Let $\rew$ be a prunable rewriting operator, and let $\mathcal Q_1$ and $\mathcal Q_2$ be two sets of queries. If $(\mathcal Q_1 \cup \mathcal Q_2) \geq \rew(\mathcal Q_1,\mathcal R)$, then $(\mathcal Q_1 \cup W_\infty(\mathcal Q_2,\mathcal R)) \geq W_\infty(\mathcal Q_1 \cup \mathcal Q_2,\mathcal R)$.
\end{lemma}

\begin{proof}
	We prove by induction on $i$ that $\mathcal Q_1  \cup W_i(\mathcal Q_2 , \mathcal{R}) \geq rew_i(\mathcal Q_1 \cup \mathcal Q_2 , \mathcal{R})$.\\
	For $i=0$, $rew_0(\mathcal Q_1 \cup \mathcal Q_2 , \mathcal{R}) = \mathcal Q_1 \cup \mathcal Q_2  = \mathcal Q_1 \cup W_0(\mathcal Q_2, \mathcal{R})$.\\
	For $i > 0$, for any $Q_{i} \in \rew_{i}(\mathcal Q_1 \cup \mathcal Q_2,\mathcal R)$, there is $Q_{i-1} \in \rew_{i-1}(\mathcal Q_1 \cup \mathcal Q_2,\mathcal R)$ such that $Q_{i} \in \rew(Q_{i-1},\mathcal R)$. By induction hypothesis, there is $Q'_{i-1} \in \mathcal Q_1 \cup W_{i-1}(\mathcal Q_2,\mathcal R)$ such that $Q'_{i-1} \geq Q_{i-1}$. Since $\rew$ is prunable,
 either $Q'_{i-1} \geq Q_i$ or there is $Q'_{i} \in \rew(Q'_{i-1}, \mathcal{R})$ such that $Q'_i \geq Q_i$. Then, there are two possibilities:
	\begin{itemize}
	\item either $Q'_{i-1} \in \mathcal Q_1$: since $\mathcal Q_1 \cup \mathcal Q_2 \geq \rew(\mathcal Q_1,\mathcal R)$, we have $\mathcal Q_1 \cup \mathcal Q_2 \geq \{Q'_i\}$ and so $\mathcal Q_1 \cup W_i(\mathcal Q_2,\mathcal R)\geq \{Q'_i\}$. 
	\item or $Q'_{i-1} \in W_{i-1}(\mathcal Q_2,\mathcal R)$: then $Q'_i \in W_{i}(\mathcal Q_2,\mathcal R)$.
	\end{itemize}
\end{proof}

Finally, the correctness of Algorithm \ref{algo-gen-fus} is based on the following loop invariants.

\begin{property}[Invariants of Algorithm \ref{algo-gen-fus}]\label{PropInv}
Let $\rew$ be a rewriting operator. After each iteration of the while loop of Algorithm \ref{algo-gen-fus}, the following properties hold:
\begin{enumerate}
\item $\mathcal Q_E \subseteq \mathcal Q_F \subseteq W_\infty(Q,\mathcal R)$;
\item $\mathcal Q_F \geq \rew(\mathcal Q_F \setminus \mathcal Q_E,\mathcal R)$;
\item if \rew is prunable then $(\mathcal Q_F \cup W_\infty(\mathcal Q_E,\mathcal R)) \geq W_\infty(Q,\mathcal R)$;
\item for all distinct $Q, Q' \in \mathcal Q_F$, $Q\not\geq Q'$ and $Q'\not\geq Q$.
\end{enumerate}
\end{property}

\begin{proof}
Invariants are proved by induction on the number of iterations of the while loop. Below $\mathcal Q_F^i$ and $\mathcal Q_E^i$ denote the value of $\mathcal Q_F$ and $\mathcal Q_E$ after $i$ iterations.
\renewcommand{\descriptionlabel}[1]%
{\hspace{\labelsep}\emph{#1}}
\begin{description}
\item [Invariant 1:] $\mathcal Q_E \subseteq \mathcal Q_F \subseteq W_\infty(Q,\mathcal R)$.
	\begin{description}
	\item [basis:] $\mathcal Q_E^0 = \mathcal Q_F^0 = \{Q\} = W_0(Q,\mathcal R)
	\subseteq W_\infty(Q,\mathcal R)$.
	\item [induction step:]
by construction, $\mathcal Q_E^{i}\subseteq \mathcal Q_F^{i}$ and $\mathcal Q_F^{i}\subseteq \mathcal Q_F^{i-1} \cup \rew(\mathcal Q_E^{i-1},\mathcal R)$. For any $Q'\in \mathcal Q_F^{i}$ we have: either $Q'\in \mathcal Q_F^{i-1}$ and then by induction hypothesis $Q'\in W_\infty(Q,\mathcal R)$; or $Q'\in \rew(\mathcal Q_E^{i-1},\mathcal R)$ and then by induction hypothesis we have $\mathcal Q_E^{i-1} \subseteq W_\infty(Q,\mathcal R)$, which implies $Q'\in W_\infty(Q,\mathcal R)$.
	\end{description}
	
\item [Invariant 2:] $\mathcal Q_F \geq \rew(\mathcal Q_F \setminus \mathcal Q_E,\mathcal R)$.
\begin{description}
\item [basis: ] 
$\rew(\mathcal Q_F^0 \setminus \mathcal Q_E^0,\mathcal R) = \rew(\emptyset,\mathcal R) = \emptyset$ and any set covers it.
\item [induction step:]

by construction, $\mathcal Q_F^{i}\geq \mathcal Q_F^{i-1} \cup \rew(\mathcal Q_E^{i-1},\mathcal R)$; since by induction hypothesis $\mathcal Q_F^{i-1}\geq \rew(\mathcal Q_F^{i-1} \setminus \mathcal Q_E^{i-1},\mathcal R)$, we have $\mathcal Q_F^{i}\geq \rew(\mathcal Q_F^{i-1} \setminus \mathcal Q_E^{i-1},\mathcal R) \cup \rew(\mathcal Q_E^{i-1},\mathcal R) = \rew(\mathcal Q_F^{i-1},\mathcal R)$.
Furthermore, by construction, $\mathcal Q_E^i=\mathcal Q_F^{i}\setminus\mathcal Q_F^{i-1}$; thus $\mathcal Q_F^{i}\setminus\mathcal Q_E^{i}\subseteq\mathcal Q_F^{i-1}$ and so $\rew(\mathcal Q_F^{i}\setminus\mathcal Q_E^{i},\mathcal R)\subseteq \rew(\mathcal Q_F^{i-1},\mathcal R)$.
Thus $\mathcal Q_F^{i}\geq\rew(\mathcal Q_F^{i}\setminus\mathcal Q_E^{i},\mathcal R)$.
\end{description}

\item [Invariant 3:] if \rew is prunable then $(\mathcal Q_F \cup W_\infty(\mathcal Q_E,\mathcal R)) \geq W_\infty(Q,\mathcal R)$.
\begin{description}
\item [basis:] $(\mathcal Q_F^0 \cup W_\infty(\mathcal Q_E^0,\mathcal R)) = (\{Q\} \cup W_\infty(\{Q\},\mathcal R)) = W_\infty(Q,\mathcal R)$.

\item [induction step:] we first show that \textit{(i):} $(\mathcal Q_F^{i} \cup W_\infty(\mathcal Q_E^{i},\mathcal R))\geq W_\infty(\mathcal Q_F^{i},\mathcal R)$, then we prove by induction that \textit{(ii):} $W_\infty(\mathcal Q_F^{i},\mathcal R)\geq W_\infty(Q,\mathcal R)$:
 \begin{itemize}
 \item[\textit{(i)}] by construction $\mathcal Q_E^{i} \subseteq \mathcal Q_F^{i}$, thus $(\mathcal Q_F^{i}\setminus\mathcal Q_E^{i})\cup \mathcal Q_E^{i}=\mathcal Q_F^{i}$, and by Invariant 2, we have $(\mathcal Q_F^{i}\setminus\mathcal Q_E^{i})\cup \mathcal Q_E^{i} \geq \rew(\mathcal Q_F^{i} \setminus \mathcal Q_E^{i},\mathcal R)$. Lemma \ref{lemma:prunable} then  entails that $((\mathcal Q_F^{i}\setminus\mathcal Q_E^{i})\cup W_\infty(\mathcal Q_E^{i},\mathcal R)) \geq W_\infty((\mathcal Q_F^{i} \setminus  \mathcal Q_E^{i})\cup \mathcal Q_E^{i},\mathcal R)$ and we can conclude since $\mathcal Q_F^{i} = (\mathcal Q_F^{i}\setminus\mathcal Q_E^{i})\cup \mathcal Q_E^{i}$.
 \item[\textit{(ii)}] \sloppy by construction, we have $\mathcal Q_F^{i}\geq \mathcal Q_F^{i-1} \cup \rew(\mathcal Q_E^{i-1},\mathcal R)$; so, by Lemma \ref{lemma:main}, we have $W_\infty(\mathcal Q_F^{i},\mathcal R)\geq W_\infty(\mathcal Q_F^{i-1} \cup \rew(\mathcal Q_E^{i-1},\mathcal R),\mathcal R) = W_\infty(\mathcal Q_F^{i-1},\mathcal R) \cup W_\infty(\rew(\mathcal Q_E^{i-1},\mathcal R),\mathcal R)$. Moreover, $\mathcal Q_E^{i-1}\subseteq\mathcal Q_F^{i-1}\subseteq W_\infty(\mathcal Q_F^{i-1},\mathcal R)$, thus $W_\infty(\mathcal Q_F^{i},\mathcal R)\geq \mathcal Q_F^{i-1} \cup \mathcal Q_E^{i-1} \cup W_\infty(\rew(\mathcal Q_E^{i-1},\mathcal R),\mathcal R) = \mathcal Q_F^{i-1} \cup W_\infty(\mathcal Q_E^{i-1},\mathcal R)$. Using \textit{(i)}, we have $W_\infty(\mathcal Q_F^{i},\mathcal R)\geq W_\infty(\mathcal Q_F^{i-1},\mathcal R)$ and conclude by induction hypothesis.
 \end{itemize}
\end{description}

\item [Invariant 4:] for all distinct $Q, Q' \in \mathcal Q_F$, $Q\not\geq Q'$ and $Q'\not\geq Q$. Trivially satisfied thanks to the properties of \cover.
\end{description}
\end{proof}

The next property states that if \rew is prunable then Algorithm \ref{algo-gen-fus} halts for each case where $W_\infty(Q,\mathcal R)$ owns a finite cover.

\begin{property}\label{haltProp}
Let \rew be a rewriting operator, $\mathcal R$ be a set of rules and $Q$ be a query. If $W_\infty(Q,\mathcal R)$ has a finite cover and \rew is prunable then Algorithm \ref{algo-gen-fus} halts.
\end{property}

\begin{proof}
Let $\mathcal Q$ be a finite cover of $W_\infty(Q,\mathcal R)$ and let $m$ be the largest $k$ for a $k$-rewriting  in $\mathcal Q$. 
 
We thus have $W_m(Q,\mathcal R)\geq \mathcal Q \geq W_\infty(Q,\mathcal R)$.
Since the operator is prunable, we have $\mathcal Q_F^i\geq W_i(Q,\mathcal R)$ for all $i \geq 0$ (which can be proved with a straightforward induction on $i$).  Thus $\mathcal Q_F^m\geq W_\infty(Q,\mathcal R)$. 
Thus, $\rew(\mathcal Q_E^m,\mathcal R)$ is covered by $Q_F^m$, and since already explored queries are taken first for the computation of a cover, we have that $\mathcal Q_E^{m+1}=\emptyset$. Hence Algorithm \ref{algo-gen-fus} halts.
\end{proof}

\begin{theorem}\label{algoCorrect}
Let \rew be a rewriting operator, $\mathcal R$ be a set of rules and $Q$ be a query. 
 If $W_\infty(Q,\mathcal R)$ has a finite cover and \rew is prunable then Algorithm \ref{algo-gen-fus} outputs this cover (up to query equivalence). 
\end{theorem}

\begin{proof}
By Property \ref{haltProp}, Algorithm \ref{algo-gen-fus} halts. By Invariant 3 from Property \ref{PropInv}, $(\mathcal Q_F^f \cup W_\infty(\mathcal Q_E^f,\mathcal R)) \geq W_\infty(Q,\mathcal R)$ where $Q_F^f$ and $Q_E^f$ denote the final values of $Q_F$ and $Q_E$ in Algorithm \ref{algo-gen-fus}. Since $Q_E^f=\emptyset$ when Algorithm \ref{algo-gen-fus} halts, we have $\mathcal Q_F^f\geq W_\infty(Q,\mathcal R)$. Thanks to Invariants 1 and 4 from Property \ref{PropInv} we conclude that $\mathcal Q_F^f$ is a cover of $W_\infty(Q,\mathcal R)$.
\end{proof}

\subsection{Preserving Soundness and Completeness}

We consider two further properties of a rewriting operator, namely soundness and completeness, with the aim of ensuring the  soundness and completeness of the obtained rewriting set within the meaning of Definition \ref{def-sound-complete-rewriting}. 

\begin{definition}[Soundness/completeness of a rewriting operator]
Let $\rew$ be a rewriting operator. $\rew$ is sound if for any set of rules $\mathcal R$, for any query $Q$, for any $Q' \in \rew(Q,\mathcal R)$, for any fact $F$, $F \models Q'$ implies that $F,\mathcal R \models Q$. $\rew$ is complete if for any set of rules $\mathcal R$, for any query $Q$, for any fact $F$ s.t. $F,\mathcal R \models Q$, there exists $Q' \in W_\infty(Q,\mathcal R)$ s.t. $F \models Q'$.
\end{definition}

\begin{property}\label{algoSound}
If $\rew$ is sound, then the output of Algorithm \ref{algo-gen-fus} is a sound rewriting set of $Q$ and $\mathcal R$.
\end{property}

\begin{proof}
 Direct consequence of Invariant 1 from Property \ref{PropInv}.
\end{proof}

Perhaps surprisingly, the completeness of the rewriting operator is not sufficient to ensure the completeness of the output rewriting set. Examples are provided in Section \ref{sec:single-prunability}. This is due to the dynamic pruning performed at each step of   Algorithm \ref{algo-gen-fus}. Therefore the prunability of the operator is also required. 

\begin{property}\label{algoComplete}
If $\rew$ is prunable and complete, then the output of Algorithm \ref{algo-gen-fus} is a complete rewriting set of $Q$ and $\mathcal R$.
\end{property}

\begin{proof}
Algorithm \ref{algo-gen-fus} returns $Q_F$ when $Q_E$ is empty. By Invariant 3 of Property \ref{PropInv}, we know that $(\mathcal Q_F \cup W_\infty(\mathcal \emptyset,\mathcal R)) \geq W_\infty(Q,\mathcal R)$. Since $W_\infty(\mathcal \emptyset,\mathcal R)) = \emptyset $, we are sure that $\mathcal Q_F \geq W_\infty(Q,\mathcal R)$.
\end{proof}

\begin{theorem}
If \rew is a sound, complete and prunable operator, and $\mathcal R$ is a finite unification set of rules, then for any query $Q$, Algorithm \ref{algo-gen-fus} outputs
a minimal (finite) sound and complete rewriting set of $Q$ with $\mathcal R$.
\end{theorem}

\begin{proof}
If $\mathcal R$ is a fus and \rew is a sound and complete operator then $W_\infty(Q,\mathcal R)$ has a finite cover. We conclude with Properties \ref{algoSound} and \ref{algoComplete} and Theorem \ref{algoCorrect}.
\end{proof}

\section{Piece-Based Rewriting}
\label{sec:pu}


As mentioned in the introduction (and illustrated in Example \ref{ex:piece1}), existential variables in rule heads induce a structure that has to be taken into account in the rewriting mechanism.  Hence the classical notion of a unifier is replaced by that of a piece-unifier \cite{blms:11}. A piece-unifier  ``unifies'' a subset $Q'$ of $Q$ with a subset $H'$ of $\head{R}$, in the sense that  the associated substitution $u$ is such that $u(Q') = u(H')$. Given a piece-unifier, $Q$ is partitioned into ``pieces'', which are minimal subsets of atoms that must processed together. More specifically, we call \emph{cutpoints}, the variables from $Q'$ that are not unified with existential variables from $H'$ (i.e., they are unified with frontier variables or constants); then a \emph{piece} in $Q$ is a minimal non-empty subset of atoms ``glued'' by variables other than cutpoints, i.e.,  connected by a path of variables that are not cutpoints. We recall below the definition of pieces given in \cite{blms:11} (where $T$ 
corresponds to the set of cutpoints).

\begin{definition}[Piece]\cite{blms:11}
Let $A$ be a set of atoms and $T \subseteq \vars{A}$. A \emph{piece} of  $A$ according to  $T$
is a minimal non-empty subset $P$ of $A$ such that, for all $a$ and $a'$ in $A$, if $a \in P$ and
$(\vars{a} \cap \vars{a'}) \not \subseteq T$, then $a' \in P$.
\end{definition}

In this paper, we give a definition of a piece-unifier based on partitions rather than substitutions, which simplifies subsequent notions and proofs. 
For any substitution $u$ from a set of variables $E_1$ to a set of terms $E_2$ associated with a piece-unifier, it holds that  $E_1 \cap E_2 = \emptyset$. We can thus assign with $u$ a \emph{partition} $P_u$ of $E_1 \cup E_2$ such that two terms are in the same class of $P_u$  if and only if they are merged by $u$; more specifically, we consider the equivalence classes of the symmetric, reflexive and transitive closure of the following relation $\sim$:  $t \sim t'$ if $u(t)=t'$. Conversely, to a partition on a set of terms $E$, such that no class contains two constants, can be assigned a substitution $u$ obtained by selecting an element of each class with priority given to constants: let $\{e_1 \ldots e_k\}$ be a class  in the partition and $e_i$ be the selected element, then for all $e_j$ with $1 \leq j  \neq i \leq k$, we set $u(e_j) = e_i$. If we consider a total order on terms, such that constants are smaller than variables, then a unique 
substitution is  obtained by taking the smallest element in each class. We call \emph{admissible partition} a partition such that no class contains two constants. 

The set of all partitions over a given set is structured in a \emph{lattice} by the ``\emph{finer than}" relation (given two partitions $P_1$ and $P_2$, $P_1$
is finer than $P_2$, denoted by $P_1 \geq P_2$,
if every class of $P_1$ is included in a class of $P_2$).\footnote{Usually, the notation $\leq$ is used to denote the relation ``finer than''. We adopt the converse convention, which is more in line with substitutions and the $\geq$ preorder on CQs.} The \emph{join} of several partitions is obtained by making the union of their non-disjoint classes until stability. The join of two admissible partitions may be a non-admissible partition. We say that several admissible partitions are \emph{compatible} if their join is an admissible partition.  Note that if the concerned partitions are relative to the same set $E$, then their join is their \emph{greatest lower bound} in the partition lattice of $E$.
 
 The following immediate property makes a link between comparable partitions and comparable substitutions. 
 
 \begin{property} Let $P_1$ and $P_2$ be two admissible partitions over the same set such that $P_1 \geq P_2$, with associated substitutions $u_1$ and $u_2$ respectively. Then there is a substitution $u$ such that $u_2 = s \circ u_1$ (i.e., $u_1$ is ``more general'' than $u_2$). 
  \end{property} 

In the following definition of a piece-unifier, we assume that $Q$ and $R$ have disjoint sets of variables. 

\begin{definition}[Piece-Unifier, Separating Variable, Cutpoint] \label{PU}  A piece-unifier of $Q$ with $R$ is  a triple $\mu=(Q',H',P_u)$, where $Q' \neq \emptyset$, $Q' \subseteq Q$, $H' \subseteq \head{R}$ and  $P_u$ is a partition on $\terms{Q'} \cup \terms{H'}$ satisfying the three following conditions: 
\begin{enumerate}
\item $P_u$ is admissible, i.e., no class in $P_u$ contains two constants;
\item if a class in $P_u$ contains an existential variable (from $H'$) then the other terms in the class are \emph{non-separating} variables from $Q'$; we call \emph{separating variables} from $Q'$, and note $\sep{Q'}$, the variables occurring in both $Q'$ and $(Q \setminus Q')$: $\sep{Q'} = \vars{Q'} \cap \vars{Q \setminus Q'}$. 
\item let $u$ be a substitution associated with $P_u$ obtained by selecting an element in each class, with priority given to constants; then $u(H') = u(Q')$. 
\end{enumerate}
We call \emph{cutpoints}, and note $cutp(\mu)$, the variables from $Q'$ that are not unified with existential variables from $H'$ (i.e., they are unified with frontier variables or constants): $cutp(\mu) = \{ x \in \vars{Q'} ~|~ u(x) \in \fr{R} \cup \consts{Q'} \cup \consts{H'}\}$.
\end{definition}

 Condition 2 in the piece-unifier definition ensures that a separating variable in $Q'$ is necessarily a cutpoint. It follows that $Q'$ is composed of pieces: indeed, an existential variable from $H'$ is necessarily unified with a non-separating variable from $Q'$, say $x$, which ensures that all atoms from $Q'$ in which $x$ occurs are also part of $Q'$. Figure \ref{fig:pu} illustrates these notions. 

\begin{figure}
\begin{center}
\includegraphics[width=10cm]{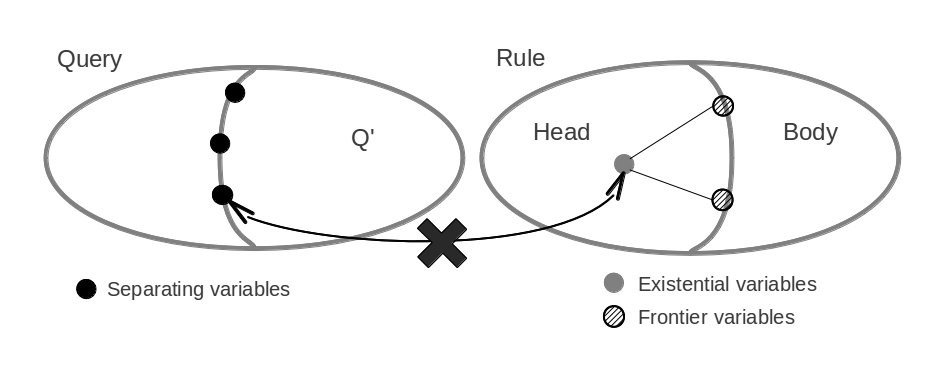}
\caption{Piece-unifier}
\end{center}
\label{fig:pu}
\end{figure}

We provide below  some examples of piece-unifiers. 

\begin{example}\label{ex:p-u1}\sloppy
Let $R = q(x) \rightarrow p(x,y)$ and  $Q = p(u,v) \wedge p(w,v) \wedge p(w,t) \wedge r(u,w)$. Let $H'=\{p(x,y)\}$. They are three piece-unifiers of $Q$ with $R$: \\
$\mu_1  = (Q'_1,H',P_u^1)$ with $Q'_1 = \{p(u,v), p(w,v)\}$ and $P_u^1= \{ \{x, u, w\}, \{y, v\} \}$\\
$\mu_2  = (Q'_2,H',P_u^2)$ with $Q'_2 = \{p(w,t)\}$ and $P_u^2= \{ \{x, w\}, \{y, t\} \}$ \\
$\mu_3  = (Q'_3,H',P_u^3)$ with $Q'_3 = \{p(u,v), p(w,v), p(w,t)\}$ and $P_u^3= \{ \{x, u, w\}, \{y, v, t\} \}$ \\
Note that $Q'_1$ and $Q'_2$ are each composed of a single piece; $Q'_3 = Q'_1 \cup Q'_2$ and  $P_u^3$ is the join of $P_u^1$ and $P_u^2$. 
\end {example}

In the previous example, $R$ has an atomic head, thus a piece-unifier of $Q'$ with $R$ actually unifies the atoms from $Q'$ and the head of $R$ into a single atom. In the general case, a piece-unifier unifies $Q'$  and a subset $H'$ of $\head{R}$ into a set of atoms, as illustrated by the next example.

\begin{example}\label{ex:p-u2}
Let  $R = q(x) \rightarrow p(x,y) \wedge p(y,z) \wedge p(z,t) \wedge r(y)$ and $Q = p(u,v) \wedge  p(v,w) \wedge r(u)$. A piece-unifier of $Q$ with $R$ is $\mu_1=(Q'_1,H'_1, P_u^1)$ with $Q'_1=\{p(u,v), p(v,w)\}$, $H'_1=\{p(x,y),p(y,z)\}$ and $P_u^1 = \{\{x, u\},\{v, y\},\{w, z\}\}$.  Another piece-unifier is $\mu_2=(Q'_2,H'_2,P_u^2)$ with $Q'_2=Q$, $H'_2=\{p(y,z),p(z,t),r(y)\}$ and $P_u^2 = \{\{u, y\},\{v, z\},\{w, t\}\}$.\\
Note that $\mu_3 = (Q'_3,H'_3, P_u^3)$ with $Q'_3 = \{p(u,v)\}$, $H'_3 = \{p(x,y)\}$ and $P_u^3 = \{\{x, u\},\{v, y\}\}$ is not a piece-unifier because the second condition in the definition of piece-unifier is not fulfilled: $v$ is a separating variable and is matched with the existential variable $y$.
\end{example} 

Then, the notions of a one-step rewriting based on a piece-unifier and of a rewriting obtained by a sequence of one-step rewritings are defined in the natural way. 

\begin{definition}[One-step Piece-Rewriting]\sloppy 
Given a piece-unifier $\mu = (Q',H',P_u)$ of $Q$ with
$R$, the \emph{one-step piece-rewriting} of $Q$ according to $\mu$, denoted $\beta(Q,R,\mu)$, is the BCQ $u(\bod{R}) \cup u(Q \setminus Q')$, where $u$ is a substitution associated with $P_u$. 
\end{definition}

We thus define inductively a \emph{$k$-step piece-rewriting} as a $(k-1)$-step piece rewriting of a one-step piece-rewriting. For any $k$, a $k$-step piece-rewriting of $Q$ is a \emph{piece-rewriting} of $Q$.

The next theorem states that piece-based rewriting is logically sound and complete. 

\begin{theorem}[basically \cite{salvat-mugnier:96}; see also \cite{blms:11}]
\label{theo:sc}
Let $\mathcal K = (F, \mathcal R)$ be a KB and $Q$ be a BCQ. Then $F, \mathcal R
\models Q$ iff there is $Q'$ a piece-rewriting of $Q$ such that $Q'\geq F$.
\end{theorem}

It follows from Theorem \ref{theo:sc} that a sound and complete rewriting operator can be based on 
piece-unifiers: we call \emph{piece-based rewriting operator}, the rewriting operator that, given $Q$ and $\mathcal R$,  outputs all the one-step piece-rewritings of $Q$ according to a piece-unifier  of $Q$ with $R \in \mathcal R$. We denote it by $\beta(Q,\mathcal R)$. 

Actually, as detailed hereafter, only most general piece-unifiers are to be considered, since the other piece-unifiers produce more specific queries. 

\begin{definition}[Most General Piece-Unifier]
Given two piece-unifiers defined on the same subsets of a query and a rule head,  $\mu_1=(Q',H',P_{u}^1)$ and $\mu_2=(Q',H',P_{u}^2)$, we say that $\mu_1$ is \emph{more general than} $\mu_2$ (notation $\mu_1 \geq \mu_2$) if
$P_{u}^1$ is finer than $P_{u}^2$
(i.e., $P_{u}^1 \geq P_{u}^2$). A piece-unifier $\mu=(Q',H',P_{u})$ is called a \emph{most general} piece-unifier if it is more general than all piece-unifiers on $Q'$ and $H'$. 
\end{definition}

\begin{property}
Let $\mu_1$ and  $\mu_2$ be two piece-unifiers with $\mu_1\geq \mu_2$. Then $\mu_1$ and $\mu_2$
have the same pieces.
\end{property}

\begin{proof}  
	 $\mu_1$ and $\mu_2$ have the same pieces iff they have the same cutpoints. It holds that $cutp(\mu_1)\subseteq cutp(\mu_2)$ since every class from $P^1_{u}$ is included in a class from $P^2_{u}$: hence a variable from $Q'$ that is in the same class as a frontier variable or a constant in $P^1_{u}$ also is in $P^2_{u}$.
	It remains to prove that $cutp(\mu_2)\subseteq cutp(\mu_1)$. Let $x$ be a cutpoint of $\mu_2$ and $P^2_{u}(x)$ be the class of $x$ in $P^2_{u}$.  Since $x$ is a cutpoint of $\mu_2$, there is a term $t$ in $P^2_{u}(x)$ that is a constant or a frontier variable. Since $P^1_{u} \geq P^2_{u}$, we know that $P^1_{u}(x) \subseteq P^2_{u}(x)$.  Let $t'$ be a term of $H'$ from $P^1_{u}(x)$ (there is at least one term of $H'$ and one term of $Q'$ in each class since the partition is part of a unifier of $H'$ and $Q'$). We are sure that $t'$ is not an existential variable because $t' \in P^2_{u}(x)$ too and an existential variable cannot be in the same class as $t$ (Condition 2 in the definition of a piece-unifier), so $t'$ is a frontier variable or a constant, hence $x$ is a cutpoint of $\mu_1$.
\end{proof}

\begin{property}\label{prop:CompUnif}
Let $\mu_1 = (Q',H',P^1_{u})$ and $\mu_2 = (Q',H',P^2_{u})$ be two piece-unifiers such that $\mu_1 \geq
\mu_2$. Then $\beta(Q, R, \mu_1) \geq \beta(Q, R, \mu_2)$.
\end{property}

\begin{proof} 
	Let $u_1$ (resp. $u_2$) be a substitution associated with $P^1_{u}$ (resp. $P^2_{u}$). Since $P^1_{u} \geq P^2_{u}$, there is a substitution $s$ such that $u_2 = s \circ u_1$ . Then $\beta(Q,R,\mu_2) = u_2(\bod{R}) \cup u_2(Q \setminus Q')$  
	$= (s \circ u_1)(\bod{R}) \cup (s \circ u_1)(Q \setminus Q') =  (s \circ u_1)(\bod{R} \cup (Q \setminus Q')) = s(u_1 (\bod{R} \cup (Q \setminus Q'))) = s(\beta(Q,R,\mu_1))$. $s$ is thus a homomorphism from $\beta(Q,R,\mu_1)$ to $\beta(Q,R,\mu_2)$, hence $\beta(Q,R,\mu_1) \geq \beta(Q,R,\mu_2)$.
\end{proof}

\smallskip 
The following lemma expresses that the piece-based rewriting operator is prunable. 

\begin{lemma}\label{propCompRewifCompQ}
If $Q_1 \geq Q_2$ then for any piece-unifier $\mu_2$ of $Q_2$ with $R$:  either (i) $Q_1
\geq \beta(Q_2, R, \mu_2)$ or (ii) there is a piece-unifier $\mu_1$ of $Q_1$ with $R$ such
that $\beta(Q_1, R, \mu_1) \geq \beta(Q_2, R, \mu_2)$.
\end{lemma}

\begin{proof}
Let $h$ be a homomorphism from $Q_1$ to $Q_2$. Let $\mu_2=(Q'_2,H'_2,P_u^2)$ be a piece-unifier of $Q_2$ with $R$, and let $u_2$ be a substitution associated with $P_u^2$. We consider two cases:

\begin{itemize}
\item[(i)] If $h(Q_1)\subseteq (Q_2 \setminus Q'_2)$, then $u_2 \circ h$ is a homomorphism from $Q_1$ to $u_2(Q_2\setminus Q'_2) \subseteq \beta(Q_2, R, \mu_2)$. Thus $Q_1 \geq \beta(Q_2, R, \mu_2)$.

\item[(ii)] Otherwise, let $Q'_1$ be the non-empty subset of $Q_1$ mapped by $h$ to
$Q'_2$, i.e., $h(Q'_1) \subseteq Q'_2$, and $H'_1$ be the subset of $H'_2$ matched by $u_2$ with $u_2(h(Q'_1))$, 
i.e., $u_2(H'_1) = u_2(h(Q'1))$. Let $P_u^1$ be the partition on $terms(H'_1) \cup terms(Q'_1)$ such that two terms are in the same class of $P_u^1$ if these terms or their images by $h$ are in the same class of $P_u^2$ (i.e., for a term $t$, we consider $t$ if $t$ is in $Q'_1$, and $h(t)$ otherwise). By construction, $(Q'_1,H'_1,P_u^1)$ is a piece-unifier of $Q_1$ with $R$. Indeed, $P_u^1$ fulfills all the conditions of the  piece-unifier definition since $P_u^2$ fulfills them.

Let $u_1$ be a substitution associated with $P_u^1$. For each class $P$ of $P_u^1$ (resp. $P_u^2$), we call selected element the unique element $t$ of $P$ such that $u_1(t)=t$ (resp. $u_2(t)=t$).  
We build a substitution $s$ from the selected elements of the classes in $P_u^1$ which are variables to the selected elements of the classes in $P_u^2$ as follows:
for any class $P$ of $P_u^1$, let $t$ be the selected element of $P$: if $t$ is a variable of $H'_1$ then $s(t) = u_2(t)$ , otherwise $s(t) = u_2(h(t))$ ($t$ occurs in $Q'_1$). Note that for any term $t$ in $P_u^1$ we have $s(u_1(t))=u_2(h(t))$.  
te that by construction of $P_u^1$, we have that for all $x \in vars(Q'_1) \cup vars(H'_1)$ $h'(u_1(x)) = u_2(h(x))$.

We build now a substitution $h'$  from $\vars{\beta(Q_1,R,\mu_1)}$ to $\terms{\beta(Q_2,R,\mu_2)}$, by considering three cases according to the part of $\beta(Q_1,R,\mu_1)$ in which the variable occurs (in $Q_1$ but not in $Q'_1$, in $\bod{R}$ but not in $H'_1$, or in the remaining part corresponding to the images of $sep(Q'_1)$ by $u_1$):
	\begin{enumerate}
		\item if $x \in \vars{Q_1} \setminus \vars{Q'_1}$, $h'(x) = h(x)$;
		\item if $x \in \vars{\bod{R}} \setminus \vars{H'_1}$, $h'(x) = u_2(x)$; 
		\item if $x \in u_1(\sep{Q'_1})$(or alternatively $x \in u_1(\fr{R}\cap \vars{H'_1})$), $h'(x) = s(x)$ ;
	\end{enumerate}

 We conclude by showing that $h'$ is a homomorphism from $\beta(Q_1,R,\mu_1) = u_1(\bod{R}) \cup u_1(Q_1 \setminus Q'_1)$ to $\beta(Q_2,R,\mu_2) = u_2(\bod{R}) \cup u_2(Q_2 \setminus Q'_2)$ with two points:
 \begin{enumerate}
 	\item $h'(u_1(\bod{R})) = u_2(\bod{R})$. Indeed, for any variable $x$ of $\bod{R}$:
\begin{itemize}
	\item either $x \in \vars{\bod{R}} \setminus \vars{H'_1}$, so $h'(u_1(x)) = h'(x) = u_2(x)$ ($u_1$ is a substitution from variables of $Q'_1\cup H'_1$), 
	\item or $x \in \fr{R} \cap \vars{H'_1}$, so $h'(u_1(x)) = s(u_1(x))=u_2(h(x))=u_2(x)$ ($h$ is a substitution from variables of $Q_1$).
\end{itemize}
 	\item $h'(u_1(Q_1 \setminus Q'_1)) \subseteq u_2(Q_2 \setminus Q'_2)$. We show that $h'(u_1(Q_1 \setminus Q'_1)) = u_2(h(Q_1 \setminus Q'_1)))$ and since $h(Q_1 \setminus Q'_1) \subseteq Q_2 \setminus Q'_2$, we have $h'(u_1(Q_1 \setminus Q'_1)) \subseteq u_2(Q_2\setminus Q'_2)$. To show that $h'(u_1(Q_1 \setminus Q'_1)) = u_2(h(Q_1 \setminus Q'_1)))$, just see that for any variable $x$ from $Q_1 \setminus Q'_1$:
\begin{itemize}
	\item either $x \in \vars{Q'_1}$, then $h'(u_1(x)) = s(u_1(x)) = u_2(h(x))$
	\item or $x \in \vars{Q_1} \setminus \vars{Q'_1}$, then $h'(u_1(x)) = h'(x)=h(x)=u_2(h(x))$ ($u_1$ is a substitution from variables of $Q'_1\cup H'_1$ and $u_2$ is a substitution from variables of $Q'_2\cup H'_2$ and $h(x)\not\in \vars{Q'_2\cup H'_2}$).
\end{itemize} 
 \end{enumerate}
 \end{itemize}
\end{proof}

We are now able to show that the piece-based rewriting operator fulfills all the desired properties introduced in section \ref{sec:generic}. 

\begin{theorem} Piece-based rewriting operator is sound, complete and prunable; this property is still true if only most general piece-unifiers are considered.
\end{theorem}

\begin{proof} Soundness and completeness follow from Theorem \ref{theo:sc}. Prunability follows from Lemma \ref{propCompRewifCompQ}. Thanks to Property \ref{prop:CompUnif}, the proof  remains true if most general piece-unifiers are considered.
\end{proof}

\section{Exploiting Single-Piece Unifiers}
\label{sec:algo}

We are now interested in the efficient computation of piece-based rewritings. We identify several sources of combinatorial explosion in the computation of the piece-unifiers between a query and a rule:
\begin{enumerate}
\item The problem of deciding whether there is a piece-unifier of a given query $Q$ with a given rule $R$ is NP-complete in the general case. NP-hardness is easily obtained by considering the case of a rule with an empty frontier: then there is a piece-unifier between $Q$ and $R$ if and only if there is a homomorphism from $Q$ to $H=\head{R}$, which is an NP-complete problem, $Q$ and $H$ being any sets of atoms. 
\item The number of most general piece-unifiers can be exponential in $|Q|$, even if the rule head $H$ is restricted to a single atom. For instance, assume that each atom of $Q$ unifies with $H$ and forms its own piece; then there may be $2 ^{|Q|}$ piece-unifiers obtained by considering all subsets of $Q$. 
\item The same atom in $Q$ may belong to distinct pieces according to distinct unifiers, as illustrated by the next example. 
\end{enumerate}

\begin{example}
 Let $Q=r(u,v) \wedge q(v)$ and $R = p(x) \rightarrow r(x,y) \wedge r(y,x) \wedge q(y)$. Atom
 $r(u,v)$ belongs to two single-piece unifiers: $(\{r(u,v),q(v)\}, \{r(x,y),q(y)\}, \{\{u,x\}, \{v,y\}\})$ and $(\{r(u,v)\},\{r(y,x)\},\{\{u,y\},\{v,x\}\})$.
For an additional example, see Example \ref{ex:p-u2}, where $p(u,v)$ and $p(v,w)$ both belong to $\mu_1$ and $\mu_2$. 
\end{example}

To cope with this complexity, one idea is to rely on \emph{single-piece} unifiers, i.e., piece-unifiers of the form $(Q',-,-)$ where $Q'$ is a single piece of $Q$.
This section is devoted to the properties of rewriting operators exploiting this notion.  Another idea is to focus on rules  with an \emph{atomic head}, which will be done in the next section. Atomic-head rules are often considered in the literature, specifically in logic programming or in deductive databases. Furthermore, any existential rule can be decomposed into an equivalent set of rules with atomic head by introducing a new predicate gathering the variables of the original head  (e.g. \cite{cali-gottlob-kifer:08,blms:09}). Hence, this restriction can be made without loss of expressivity. Considering atomic-head rules does not  simplify the definition of a piece-unifier  in itself, but its computation: there is now a unique way of associating any atom from $Q$ with the head of a rule.  Thus, deciding whether there is a piece-unifier of $Q$ with a rule can be done in linear time with respect to the size of $Q$ (which tames complexity source 1 in the above list) and each atom belongs to a single piece (
see complexity source 3), thus the set of all single-piece unifiers of $Q$ with a rule can be computed in polynomial time.

In this section, we show that the rewriting operator based on single-piece (most general) unifiers is sound and complete. However, perhaps surprisingly, it is not prunable, which prevents to use it in the generic algorithm. To recover prunability, we will define the aggregation of single-piece unifiers, which provides us with a new rewriting operator, which has all the desired properties and generates less rewritings than the standard piece-unifier.  Note however that this will not completely remove the second complexity source (i.e., the exponential number of unifiers to consider) since the number of agregations of single-piece unifiers can still be exponential in the size of $Q$, even with atomic-head rules.

\subsection{Single-Piece Based Operator}
\label{sec:single-prunability}
 
As expressed by the following theorem, (most general) single-piece unifiers provide a sound and complete operator.

\begin{theorem}\label{CompleteWithSPU}
\label{thm-sgp}
Given a BCQ $Q$ and a set of rules $\mathcal R$, the set of rewritings of $Q$ obtained by
considering exclusively most general single-piece unifiers is sound and complete. 
\end{theorem}

\begin{proof}
See Appendix.
\end{proof}

The proof of this theorem is given in Appendix since it is not reused hereafter. Indeed, the restriction to single-piece unifiers is not compatible with selecting most general rewritings  at each step, as done in Algorithm
\ref{algo-gen-fus}.  We present below some examples that illustrate this incompatibility. 

\begin{example} [Basic example] Let $Q = p(y,z) \wedge p(z,y)$ and $R =  r(x,x) \rightarrow
p(x,x)$. There are two single-piece unifiers of $Q$ with $R$, $\mu_1 = (\{p(y,z)\}, \{p(x,x)\}, \{ \{x,y,z\} \})$
and $\mu_2 = ( \{p(z,y)\} , \{p(x,x)\}, \{ \{x,y,z\} \})$, which yield 
the same rewriting, e.g. $Q_1 =  r(x,x)\wedge p(x,x)$. There is also a two-piece
unifier $\mu=(Q, \{p(x,x)\},\{ \{x,y,z\} \})$, which yields e.g. $Q' = r(x,x)$. A query equivalent to $Q'$ can be
obtained from $Q_1$ by a further
single-piece unification. Now, assume that we restrict unifiers to single-piece unifiers \emph{and} keep
 most general rewritings at each step. Since $Q \geq
Q_1$, $Q_1$ is not kept, so $Q'$ will never be generated,
whereas it is incomparable with $Q$.
\end{example}

Concerning the preceding example, one may argue that $u_1(Q)$ is redundant (and the same holds for  $u_2(Q)$), and that the problem would be solved by computing  $u_1(Q) \setminus u_1(Q')$ instead of  $u_1(Q \setminus Q')$ and making $u_1(Q)$ non-redundant (i.e., equal to $p(x,x)$) before computing  $u_1(Q) \setminus u_1(Q')$, which would then be empty. 
However, the problem goes deeper, as the next examples show it. 

\begin{example}[Ternary predicates] \sloppy
Let $Q = r(u,v,w) \wedge  r(w,t,u)$ and  $R = p(x,y) \rightarrow r(x,y,x)$. Again, there are two single-piece unifiers of $Q$ with $R$: $\mu_1 = ( \{r(u,v,w)\},$ $\{r(x,y,x)\}, \{\{u,w,x\},\{v,y\}\})$ and 
$\mu_2 = (\{r(w,t,u)\}, \{r(x,y,x)\}, \{\{u,w,x\},$ 
$\{t,y\}\})$. One obtains two rewritings more specific than $Q$, e.g.  $Q_1 = p(x,y) \wedge r(x,v,x), $ and $Q_2 = p(x,y) \wedge r(x,t,x)$, which are isomorphic. There is also a two-piece unifier $(Q, \{r(x,y,x)\}, \{\{u,w,x\},\{v,t,y\} \})$, which yields e.g. $p(x,y)$. If we remove $Q_1$ and $Q_2$, no query equivalent to $p(x,y)$ can be generated.
\end{example}

\begin{figure}[ht]
\centerline{\includegraphics[width=12cm]{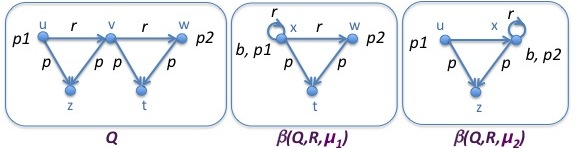}}
\caption{\label{fig-example-mono-pieces}  The queries in Example \ref{example-binary}}
\end{figure}

\begin{example}[Very simple rule]  \label{example-binary}This example has two interesting characteristics: (1) it uses unary/binary predicates only (2) it uses a very simple rule expressible with any lightweight description logic, i.e., a linear existential rule where no variable appears twice in the head or the body.
Let $Q = r(u,v) \wedge r(v,w) \wedge p(u,z) \wedge p(v,z) \wedge p(v,t) \wedge p(w,t) \wedge p_1(u) \wedge p_2(w)$ (see Figure \ref{fig-example-mono-pieces}) and $R = b(x) \rightarrow p(x,y)$. Note that $Q$ is not redundant. There are two single-piece unifiers of $Q$ with $R$, say $\mu_1$ and $\mu_2$, with pieces $Q'_1 = \{p(u,z),p(v,z)\}$ and 
$Q'_2 = \{p(v,t),p(w,t)\}$ respectively. The obtained queries are pictured in Figure \ref{fig-example-mono-pieces}. These queries are both more specific than $Q$.  The removal would prevent the generation of a query equivalent to $r(x,x)\wedge p_1(x)\wedge p_2(x)\wedge b(x)$, which could be generated from $Q$ with a two-piece unifier. 
\end{example}

\begin{property} The single-piece-based operator is not prunable.
\end{property}

\begin{proof}  Follows from the above examples.
\end{proof}

 By Theorem 5 and Property \ref{propMGRifMGSinglePieceRew}, one can show that the conclusion of Lemma \ref{lemma:main} is valid for single-piece unifiers, even though they are not prunable. This justifies that Lemma \ref{lemma:main} is not enough to prove the correctness of Algorithm \ref{algo-gen-fus}.
However, single-piece unifiers can still be used as an algorithmic brick to compute more complex piece-unifiers, as shown in the next subsection.

\subsection{Aggregated-Piece Based Operator}

We first explain the ideas that underline aggregated single-piece unifiers. Let us consider the set of single-piece unifiers naturally associated  with a piece-unifier $\mu$. If we apply successively each of these underlying
single-piece unifiers, we may obtain a CQ strictly more general than $\beta(Q,R, \mu)$, as illustrated in the next example.

\begin{example}\label{ex-aggreg} Let $R = p(x,y) \rightarrow q(x,y)$ and $Q = q(u,v)\wedge  r(v,w)\wedge q(t,w)$. 
Let $\mu = (Q', H', P_u)$ be a piece-unifier  of $Q$
with $R$ with 
$Q' = \{q(u,v),q(t,w)\}$, $H' = \{q(x,y)\}$ and $P_u = \{\{u,t,x\},\{v,w,y\}\}$. $\beta(Q,R,\mu) = p(x,y)\wedge  r(y,y)$.
$Q'$ has two pieces w.r.t. $\mu$:
$P_1=\{q(u,v)\}$ and 
$P_2=\{q(t,w)\}$. 
 If we successively computing the rewritings with the underlying single-piece unifiers $\mu_{P_1}$ and
$\mu_{P_2}$,  we
obtain $\beta(\beta(Q,R,\mu_{P_1}),R,\mu_{P_2}) = \beta(p(x,y)\wedge  r(y,w)\wedge
q(t,w),R,\mu_{P_2}) = p(x,y)\wedge  r(y,y')\wedge p(x',y')$, which is strictly more general than $\beta(Q, R, \mu)$.
\end{example}

Given a set $\mathcal U$ of ``compatible'' single-piece unifiers of a query $Q$ with a rule (the notion of ``compatible'' will be formally defined below), we can thus distinguish between the usual piece-unifier performed on the union of the pieces from the unifiers in
$\mathcal U$ and an ``aggregated unifier'' that would correspond to a sequence of applications of the single piece-unifiers in $\mathcal U$. This latter unifier is more interesting than the piece-unifier because, as illustrated by Example \ref{ex-aggreg}, it avoids generating some rewritings which are too specific. We will thus rely on the aggregation of single-piece unifiers to recover prunability. 

Note that, in this paper, we combine single-piece unifiers of the \emph{same} rule whereas in \cite{rr-13-klmt} we consider the possibility of combining unifiers of distinct rules (and thus compute rewritings from distinct rules in a single step).
We keep here the definitions introduced in \cite{rr-13-klmt}, while pointing out that, in the context of this paper, the rules $R_1 \ldots R_k$ are necessarily copies of the same rule $R$. 

\begin{definition}[Compatible Piece-Unifiers]
Let $\mathcal U = \{\mu_1=(Q'_1, H'_1, P_1)  \ldots \mu_k= (Q'_k, H'_k, P_k) \}$ be a set of piece-unifiers of $Q$ with rules $R_1 \ldots R_k$ respectively, where all $R_i$ have
disjoint sets of variables
(hence, for all $1 \leq i, j \leq k, i \neq j, \vars{H'_i} \cap  \vars{H'_j} = \emptyset$).
 $\mathcal U$ is said to be \emph{compatible} if (1) all $Q'_i$ and $Q'_j$ are pairwise disjoint; (2) the join of $P_1 \ldots P_k$ is admissible.
\end{definition}

\begin{definition}[Aggregated unifier]\sloppy
Let $\mathcal U = \{\mu_1=(Q'_1, H'_1, P_1),  \ldots, \mu_k= (Q'_k, H'_k, P_k) \}$ be a compatible set of piece-unifiers of $Q$ with rules $R_1 \ldots R_k$.
An \emph{aggregated unifier} of $Q$ with $R_1 \ldots R_k$ w.r.t. $\mathcal U $ is $\mu = (Q',H',P)$ where: (1) $Q' = Q'_1 \cup \ldots \cup Q'_k$; (2) $H' = H'_1 \cup \ldots \cup H'_k$; (3) $P$ is the join of $P_1 \ldots P_k$. It is said to be \emph{single-piece} if all the piece-unifiers of $\mathcal U$ are single-piece. It is said to be \emph{most general} if all the piece-unifiers of $\mathcal U$ are most general.
\end{definition}

  \begin{definition}[Aggregation of a set of rules]
  The aggregation of a set of rules $\mathcal{R} = \{ R_1 \ldots R_k \}$, denoted by $R_1 \diamond \ldots \diamond R_k$, is the rule $\bod{R_1} \wedge \ldots \wedge \bod{R_k} \rightarrow \head{R_1} \wedge \ldots \wedge \head{R_k}$, where it is assumed that all rules have disjoint sets of variables. 
 \end{definition}

 \begin{property} \label{prop_aggr=pu}
 Let $Q$ be a BCQ
 and $\mathcal U = \{\mu_1=(Q'_1, H'_1, P_1)  \ldots \mu_k= (Q'_k, H'_k, P_k) \}$ be a compatible set of piece-unifiers of $Q$ with $R_1 \ldots R_k$.
Then the aggregated unifier of $\mathcal U$ is a piece-unifier of $Q$ with the aggregation of $\{ R_1 \ldots R_k \}$.
 \end{property}

 \begin{proof}
 We show that the aggregated unifier $\mu = (Q', H', P_u)$ of $\mathcal U$ satisfies the conditions of the definition of a piece-unifier. Condition 1 is fulfilled since by definition of compatibility, the join of $P_1 \ldots P_k$ is admissible. Condition 2 is satisfied too, because since $P_1 \ldots P_k$ satisfy it, so does their join. Indeed, if a class contains an existential variable, it cannot be merged with another by aggregation because its other terms are non-separating variables, hence do not appear in other classes. Concerning the last condition, for all $1 \leq i \leq k$ we have $u_i(H'_i) = u_i(Q'_i)$ where $u_i$ is a substitution associated with $P_i$. Since $Q' = \bigcup_{i=1}^k Q'_i$ and $H' = \bigcup_{i=1}^k H'_i$ we are sure that for any  substitution $u$ associated with $P_u$ we have $u(H') = u(Q')$.
 \end{proof}
 
 The rewriting associated with an aggregated unifier $\mu$ can thus be defined as $\beta(Q, R_1 \diamond \ldots \diamond R_k, \mu)$. 
  It is equivalent to the rewriting obtained by applying the single-piece unifiers one after the other.

\begin{example} 
 Consider again Example \ref{ex-aggreg}. Let $R' = p(x',y') \rightarrow q(x',y')$ be a copy of $R$. Then the aggregation $R \diamond R'$ is the rule $ p(x,y) \wedge p(x',y') \rightarrow q(x,y) \wedge q(x',y')$.  Let $\mathcal U = \{\mu_{P_1}, \mu_{P_2}\}$ where $\mu_{P_1} =(\{q(u,v)\}, \{q(x,y)\}, \{\{u,x\}, \{v,y\}\})$ and  $\mu_{P_2} =(\{q(t,w)\}, \{q(x',y')\}, \{\{t,x'\}, \{w,y'\}\})$ . The aggregated unifier of $Q$ with $R, R'$ w.r.t.  $\mathcal U $ is $(\{q(u,v),q(t,w)\}, \{q(x,y),q(x',y')\}, 
 \{\{u,x\}, \{v,y\}, \{t,x'\}, \{w,y'\}\})$. The associated rewriting of $Q$ is $p(x,y)\wedge  r(y,y')\wedge p(x',y')$.
  \end{example}

Note that, if we assumed, in the definition of an aggregated unifier, that $R_1 = \ldots R_k = R$ (and in particular have the same variables), then the aggregated unifier would be the usual piece-unifier, and the aggregation of $R_1 \ldots R_k$ would be exactly $R$ after removal of duplicate atoms.  In other words, to build a standard piece-unifier of $Q$ with $R$ we consider partitions of $\terms{Q} \cup \terms{\head{R}}$, while in the aggregation operation we consider  $\terms{Q} \cup \bigcup_{i=1}^k \terms{\head{R_i}}$, where $k$ is the number of single-piece unifiers of $Q$ with $R$ and each $R_i$ is safely renamed from $R$. 

The next property shows that from any piece-unifier $\mu$, one can build a most general single-piece aggregated unifier, which produces a rewriting more general than the one produced by $\mu$. 

\begin{property} \label{prop:aggreg}
	For any piece-unifier $\mu$ of $Q$ with $R$, there is a most general single-piece aggregated unifier $\mu_\diamond$ of $Q$ with $R_1 \ldots R_k$ copies of $R$ such that $\beta(Q,R_1 \diamond \ldots \diamond R_k, \mu_\diamond) \geq  \beta(Q,R, \mu)$. 
\end{property}

\begin{proof}
Let $Q'_1,\ldots,Q'_k$ be the pieces of $Q'$ according to $\mu=(Q',H',P_u)$ and let $u$ be a substitution associated to $P_u$. 
Let $R_1 \ldots R_k$ be safely renamed copies of $R$. Let $h_i$ denote the variable renaming used to produce $R_i$ from $R$.  
Let  $\mathcal U = \{\mu_1=(Q'_1,H'_1,P_u^1), \dots, \mu_k=(Q'_k,H'_k,P_u^k)\}$ be a set of piece-unifiers of $Q$ with $R_1, \dots, R_k$ built as follows for all $i$:  
\begin{itemize}
\item $H'_i$ is the image by $h_i$ of the subset of $H'$ unified by $u$ with $Q'_i$
\item let $h_i(P_u)$ be the partition built from $P_u$ by replacing each $x \in \vars{H'}$ by $h_i(x)$; then $P_u^i$ is obtained from $h_i(P_u)$  by (1) restricting it to the terms of $Q'_i$ and $H'_i$  (2) refining it as much as possible while keeping the property that $u_i(H'_i) = u_i(Q'_i)$, where $u_i$ is a substitution associated with the partition. 
\end{itemize}
For any $\mu_i=(Q'_i,H'_i,P_u^i)$ we immediately check that:
\begin{enumerate}
	\item $\mu_i$ is a most general piece-unifier.
	\item $\mu_i$ is a single-piece unifier.
	\item $\forall \mu_j \in \mathcal U$, $\mu_i \neq \mu_j$, $\mu_j$ and $\mu_i$ are compatible.
\end{enumerate}

Let $\mu_\diamond = (Q'_\diamond,H'_\diamond, P_u^\diamond)$ be the aggregated unifier of $Q$ with $R_1, \dots, R_k$ w.r.t. $\mathcal U$.
Note that $Q'_\diamond=Q'$.  The above properties fulfilled by any $\mu_i$ from $\mathcal U$ ensure that $\mu_\diamond$ is a most general single-piece aggregated unifier.

We note $ R_{\diamond} = R_1 \diamond \ldots \diamond R_k$. 
It remains to prove that $\beta(Q,R_{\diamond}, \mu_{\diamond}) \geq  \beta(Q,R, \mu)$.
Let $u_\diamond$ be a substitution associated with $P_u^\diamond$. For each class $P$ of $P_u$ (resp. $P_u^\diamond$), we call selected element the unique element $t$ of $P$ such that $u(t)=t$ (resp. $u_\diamond(t)=t$).

We build a substitution $s$ from the selected elements in $P_u^\diamond$ which are variables to the selected elements in $P_u$ as follows:
for any class $P$ of $P_u^\diamond$, let $t$ be the selected element of $P$: if $t$ is a variable of $Q'$ then $s(t) = u(t)$; else $t$ is a variable of a $H'_i$: then $s(t) = u(h_i^{-1}(t))$. Note that for any term $t$ in $P_u^\diamond$, there is a variable renaming $h_i$ such that $s(u_\diamond(t))=u(h_i^{-1}(t))$ (if $t$ is a constant or a variable from $\vars{Q}$ then any $h_i$ can be chosen).  

We build now a substitution $h$ from $\vars{\beta(Q,R_\diamond,\mu_\diamond)}$ to $\terms{\beta(Q,R,\mu)}$, by considering three cases according to which part of $\beta(Q,R_\diamond,\mu_\diamond)$ the variable occurs (in $Q$ but not in $Q'$, in $\bod{R_i}$ but not in $H'_i$, or in the remaining part corresponding to the images of $sep(Q')$ by $u_\diamond$):
	\begin{enumerate}
		\item if $x \in \vars{Q} \setminus \vars{Q'}$, $h(x) = x$;
		\item if $x \in \vars{\bod{R_i}} \setminus \vars{H'_i}$, $h(x) = h_i^{-1}(x)$; 
		\item if $x \in u_\diamond(\sep{Q'})$(or alternatively $x \in u_\diamond(\fr{R_\diamond}\cap \vars{H'_\diamond})$), $h(x) = s(x)$ ;
	\end{enumerate}

 We conclude by showing that $h$ is a homomorphism from $\beta(Q,R_\diamond,\mu_\diamond) = u_\diamond(\bod{R_1} \cup \dots \cup \bod{R_k}) \cup u_\diamond(Q \setminus Q')$ to $\beta(Q,R, \mu) = u(\bod{R}) \cup u(Q\setminus Q')$ with two points:
 \begin{enumerate}
 	\item for all $i$, $h(u_\diamond(\bod{R_i})) = u(\bod{R})$. Indeed, for any variable $x  \in \vars{\bod{R_i}}$:
\begin{itemize}
	\item either $x \in \vars{\bod{R_i}} \setminus \vars{H'_i}$, so $h(u_\diamond(x)) = h(x) = h_i^{-1}(x)= u(h_i^{-1}(x))$ ($u$  does not substitute the variables in $\vars{\bod{R}} \setminus \vars{H'}$), 
	\item or $x \in \fr{R_i} \cap \vars{H'_i}$, so $h(u_\diamond(x)) = s(u_\diamond(x))=u(h_i^{-1}(x))$;
\end{itemize}

 	\item $h(u_\diamond(Q \setminus Q')) = u(Q \setminus Q')$. Indeed, for any variable $x \in \vars{Q \setminus Q'}$:
\begin{itemize}
	\item either $x \in \vars{Q'}$, then $h(u_\diamond(x)) = s(u_\diamond(x)) = u(h_i^{-1}(x)) = u(x)$ ($h_i^{-1}$ does not substitute the variables in $Q$),
	\item or $x \in \vars{Q} \setminus \vars{Q'}$, then $h(u_\diamond(x)) = h(x) = x = u(x)$ ($u_\diamond$ and $u$ do not substitute the variables in $\vars{Q} \setminus \vars{Q'}$).
\end{itemize} 
 \end{enumerate}
\end{proof}

We call \emph{single-piece aggregator} the rewriting operator that computes the set of one-step rewritings of a query $Q$ by considering all the most general single-piece aggregated unifiers of $Q$.

 \begin{theorem} 
The single-piece aggregator is sound, complete and prunable.
 \end{theorem}

 \begin{proof}
 Soundness comes from Property \ref{prop_aggr=pu} and from the fact that for any set of rules $\mathcal{R}$, let the rule $R$ be its aggregation, one has $\mathcal{R}\models R$.
 Completeness and prunability rely on the fact that the piece-based rewriting operator fulfills these properties and the fact that for any queries $Q$ and $Q'$ and any rule $R$, if $Q' = \beta(Q,R,\mu)$, where $\mu$ is a piece-unifier, then the query $Q''$ obtained with the single-piece aggregator corresponding to $\mu$  is more general than $Q'$, as expressed by Property \ref{prop:aggreg}.  
 \end{proof}

\section{Implementation and Experiments}

As explained in Section \ref{sec:algo}, we now restrict our focus to rules  with an \emph{atomic head}. We first detail algorithms for computing all the most general single-piece unifiers of a query $Q$ with a rule $R$ and explain how we use them to compute all single-piece aggregators. Then we report first experiments.

\subsection{Computing single-piece unifiers and their aggregation}

When a rule $R$ has an atomic head, it holds that every atom in $Q$ participates in \emph{at most one }most general single-piece unifier of $Q$ with $R$ (up to bijective variable renaming). This is is a corollary of the next property.

\begin{property}\label{corollary}
Let $R$ be an atomic-head rule and $Q$ be a BCQ.  For all atom $a\in Q$, there is at most one $Q'\subseteq Q$ such that $a\in Q'$ and $Q'$ is a piece for a piece-unifier of $Q$ with $R$.
\end{property}

\begin{proof} 
We prove by contradiction that two single-piece unifiers cannot share an atom of $Q$. Assume there are $Q_1'\subseteq Q$ and $Q_2' \subseteq Q$ such that $Q_1'\neq Q_2'$ and $Q_1'\cap Q_2'\neq\emptyset$, and $\mu_1 = (Q_1',H,P_{u}^1)$ and $\mu_2 = (Q_2',H, P_{u}^2)$  two single-piece-unifiers of $Q$ with $R$, with $H = \head{R}$.
Since $Q_1'\neq Q_2'$, one has $Q_1'\setminus Q_2'\neq\emptyset$ or  $Q_2'\setminus Q_1'\neq\emptyset$. Assume $Q_1'\setminus Q_2'\neq\emptyset$. 
Let $A=Q_1'\cap Q_2'$ and $B=Q_1'\setminus A$.  
There is at least one variable $x \in \vars{A}\cap \vars{B}$ such that there is an existential variable $e$ of $\head{R}$ in the class of $P_u^1$ containing $x$ (otherwise $\mu_1$ has more than one piece).
Since $H$ is atomic, there is a unique way of associating any atom with $H$, thus the class of $P_u^2$ containing $x$ contains also $e$. It follows that $Q'_2$ is not a piece since one atom of $A$ and one atom of B share $x$ unified with an existential variable in $\mu_2$ while $A$ is included in $Q'_2$ and $B$ is not.
\end{proof}

To compute most general single-piece unifiers, we first introduce the notion of the unification of a set of atoms with the head of a rule. This notion is an adaptation of the classical logical unification that takes existential variables into account. To define a piece-unifier, the set of atoms has to satisfy an additional constraint on its separating variables. 

\begin{definition}[Partition by Position]
Let $A$ be a set of atoms with the same predicate $p$. The \emph{partition by position} associated with $A$, denoted by $P_p(A)$, is the partition on $\terms{A}$ such that two terms of $A$ appearing in the same position $i$ ($1 \leq i \leq arity(p)$) are in the same class of $P_p(A)$.
\end{definition}

\begin{definition}[Unifiability]
Let $R$ be an atomic head rule and let $A$ be a set of atoms with same predicate $p$ as $\head{R}$. $A$ is \emph{unifiable} with $R$ if no class of $P_p(A \cup \head{R})$ 
contains two constants, or contains two existential variables of $R$, or contains a constant and an existential variable of $R$, or contains an existential variable of $R$ and a frontier variable of $R$.
\end{definition}

\begin{definition}[Sticky Variables]
Let $Q$ be a BCQ, $R$ be an atomic head rule and $Q'$ be a subset of atoms in $Q$ with the same predicate $p$ as $\head{R}$. The \emph{sticky variables} of $Q'$ with respect to $Q$ and $R$, denoted by $sticky(Q')$, are the separating variables of $Q'$ that occur in a class of $P_p(Q' \cup \head{R})$ 
 containing an existential variable of $R$.
\end{definition}

The following property follows from the definitions:

\begin{property}

Let $Q$ be a BCQ, $R$ be an atomic head rule, and $Q'$ a subset of atoms in $Q$ with the same predicate $p$ as $\head{R}$. Then $\mu=(Q', \head{R}, P_p(Q' \cup \head{R}))$ is a piece-unifier of $Q$ with $R$ iff $Q'$ is unifiable with $\head{R}$ and $sticky(Q') = \emptyset$.

\end{property}

The fact that an atom from $Q$ participates in at most one most general single-piece unifier suggests an incremental method to compute these unifiers. 
Assume that the head of $R$ has predicate $p$. We start from each atom $a \in Q$ with predicate $p$ and compute the subset of atoms from $Q$ that would necessarily belong to the same piece as $a$; more precisely, at each step, we build $Q'$ such that $Q'$ and $\head{R}$ can be unified, then check if $sticky(Q')=\emptyset$.  If there is a piece-unifier of $Q'$ built in this way with $\head{R}$, all atoms in $Q'$ can be removed from $Q$ for the search of other single-piece unifiers; otherwise, $a$ is removed from $Q$ for the search of other single-piece unifiers but the other atoms in $Q'$ still have to be taken into account. Note that in both cases, the notion of separating variables is still relative to the original $Q$.

\begin{example} Let $R = q(x) \rightarrow p(x,y)$ and $Q = p(u,v)\wedge p(v,t)$. Let us start from $p(u,v)$: this atom is unifiable with $\head{R}$ and $p(v,t)$ necessarily belongs to the same piece-unifier (if any) because $v\in sticky(\{p(u,v)\})$ ($v$ is in the same class that the existential variable $y$); however, $\{p(u,v), p(v,t)\}$ is not unifiable with $\head{R}$ because, since $v$ occurs at the first and at the second position of a $p$ atom, $x$ and $y$ should be unified, which is not possible since $y$ is an existential variable; thus $p(u,v)$ does not belong to any piece-unifier with $R$. However, $p(v,t)$ still needs to be considered. Let us start from it: $p(v,t)$ is unifiable with $\head{R}$ and forms its own piece because sticky(\{p(v,t)\}) is empty ($t$ is in the same class that the existential variable $y$ but is not shared with another atom). There is thus one (most general) piece-unifier of $Q$ with $R$, namely $(\{p(v,t)\}, \{p(x,y)\}, \{\{v,x\},\{t,y\}\})$.
\end{example} 

More precisely, Algorithm \ref{algoMGSPU} first builds the subset $A$ of atoms in $Q$ with the same predicate as $\head{R}$. While $A$ has not been emptied, it initializes a set $Q'$ by picking an atom $a$ in $A$,  then repeats the following steps: 
\begin{enumerate}
\item check if $Q'$ is unifiable with $\head{R}$; else, the attempt with $a$ fails; 
\item check if $sticky(Q')=\emptyset$; if so, it is a single-piece unifier and all the atoms in $Q'$ are removed from $A$;  
\item otherwise, the algorithm tries to extend $Q'$ with all the atoms in $Q$ containing a variable from $sticky(Q')$; if these atoms are in $A$, $Q'$ can grow, otherwise the attempt with $a$ fails. 
\end{enumerate}

\begin{algorithm}[ht]
\KwData{a CQ $Q$ and an atomic-head rule $R$}
\KwResult{the set of most general single-piece unifiers of $Q$ with $R$}
\Begin{
  $U\leftarrow \emptyset$; \tcp{ resulting set}
  $A \leftarrow \{a\in Q \mid predicate(a)=predicate(\head{R})\}$\; 
  \While{$A\neq\emptyset$}{
  	$a\leftarrow$ choose an atom in $A$ \;
  	$Q'\leftarrow\{a\}$ \;
 \While{$Q' \subseteq A$ and $unifiable(Q',\head{R})$ and $sticky(Q') \neq \emptyset$}{
  	    $Q'\leftarrow Q' \cup \{a'\in Q \mid a' ~\textrm{contains a variable in ~}sticky(Q')\}$ \;
}
  	\If{$Q' \subseteq A$ and $unifiable(Q',\head{R})$}{
  		$U\leftarrow U\cup \{(Q',\head{R},P_p(Q'\cup\head{R}))\}$ \;
		$A\leftarrow A\setminus Q'$
		}
	\Else{
		$A\leftarrow A\setminus \{a\}$
		}
	}
  \Return{$U$}
 }
\caption{Computation of all most general single-piece unifiers\label{algoMGSPU}}
\end{algorithm}

Now, to compute the set of single-piece aggregators of $Q$ with $R$, we proceed as follows:

\begin{enumerate}
\item Compute all (most general) single-piece unifiers of $Q$ with $R$: \\
$U_1 = \{\mu_1, \ldots, \mu_k\}$;
\item For $i$ from $2$ to the greatest possible rank (as long as $U_i$ is not empty): let $U_i$ be the set of all $i$-unifiers obtained by aggregating an $i-1$-unifier from $U_{i-1}$ and a single-piece unifier from $U_1$. 
\item Return the union of all the $U_i$ obtained. 
\end{enumerate}

\subsection{Experiments and Perspectives}

The generic breadth-first algorithm, instantiated with the  rewriting operator described in the preceding section, has been implemented in Java. 
 First experiments were led on sets of existential rules obtained by translation from ontologies 
expressed in the description logic DL-Lite$_{\mathcal R}$ and  developed in several research projects, namely ADOLENA (A), STOCKEXCHANGE (S), UNIVERSITY (U) and VICODI (V). See \cite{icde-11-gop} for more details. The obtained rules have atomic head and body, which corresponds to the linear Datalog+/- fragment. The associated queries were  generated by the tool Sygenia \cite{isg:12}. Sygenia  provided us with 114, 185, 81 and 102 queries for ontologies A, S, U and V respectively. In \cite{klmt:12} we compared with other systems concerning the size of the output and pointed out that none of the existing systems output a complete set of rewritings. However, beside the fact that these systems have evolved since then, one can argue that
 the size of the rewriting set should not be a decisive criterion (indeed, assuming that the systems are sound and complete, a minimal rewriting set can be obtained by selecting most general elements, see Theorem \ref{prop-mgr}). Therefore, other criteria have to be taken into account, such as the running time or the total number of BCQs built during the rewriting process. 
 
Table \ref{tab:expe} presents for each ontology the total number of generated rewritings, i.e., the sum of the number of generated BCQs for all the queries associated with a given ontology  (\# \emph{generated} column). This number can be compared with the total number of output rewritings, i.e.,  the sum of the cardinalities of the final output sets for all the queries associated with a given ontology (\# \emph{output} column).  The generated rewritings are all the rewritings built during the rewriting process (excluding the initial query and possibly including some multi-occurrences of the same rewritings). Since we remove the subsumed rewritings at each step of the breadth-first algorithm, only some
of the generated rewritings at a given step are explored at the next step. We can see that the number of generated queries can be huge with respect to the size of the output, specially for Ontology A. 

Concerning the running time, our implementation is yet far from being optimized. Moreover, our system is able to process any kind of existential rules, which involves complex mechanisms. Much time could be saved by processing specific kinds of rules in a specific way.  In particular, a large part of available ontologies is actually composed of concept and role hierarchies. For instance, 64\%, 31\%, 47\% and 90\%  of the rules in ontologies A, S, U and V respectively, express atomic concept or atomic role inclusions. 
By simply processing these sets of rules as preorders, we can dramatically decrease the running time and the number of generated queries. First experiments with ontology A show that the running time is decreased by a factor of 74 approximatively, and the number of generated queries is divided by 37.

Further work includes processing specific kinds of rules in a specific way while keeping a system able to process any set of existential rules. Other optimizations could be implemented such as exploiting dependencies between rules to select the rules to be considered at each step. Moreover, the form of the considered output itself, i.e., a union of conjunctive queries, leads to combinatorial explosion.  Considering semi-conjunctive queries instead of conjunctive queries as in \cite{ijcai-13-t} can save much with respect to both the running time and the size of the output, without compromising the efficiency of query evaluation; to generate semi-conjunctive queries, the piece-based rewriting operator is combined with query factorization techniques.  Finally, further experiments should be performed on more complex ontologies. However, even if slightly more complex ontologies could be obtained by translation from decription logics, real-world ontologies that would take advantage of the expressiveness of existential rules, as 
well as associated queries, are currently lacking.

\begin{table}
\begin{center}
\begin{tabular}{| c | c | c | c | c |}
\hline
 \  rule base \  & \  \# output \  & \  \# generated   \\ \hline
A & 3209 & 146 523  \\ \hline
S & 557 & 6515  \\ \hline
U & 486 & 2122  \\ \hline
V & 2694 & 5318  \\ \hline
\end{tabular}
\vspace{0.2cm}
\caption{ Generated Queries with the Single-Piece Aggregator \label{tab:expe}}
\end{center}
\end{table}

 \label{sec:expe}

\paragraph{Acknowledgments.} We thank Giorgio Orsi for providing us with rule versions of the ontologies. This work was partially funded by the ANR project PAGODA (ANR-12-JS02-007-01).

\bibliographystyle{alpha}
\bibliography{biblio}

\newcommand{\etalchar}[1]{$^{#1}$}
\begin{thebibliography}{PUHM09}

\bibitem[BLMS09]{blms:09}
J.-F. Baget, M.~Lecl\`ere, M.-L. Mugnier, and E.~Salvat.
\newblock Extending decidable cases for rules with existential variables.
\newblock In {\em IJCAI'09}, pages 677--682, 2009.

\bibitem[BLMS11]{blms:11}
J.-F. Baget, M.~Lecl{\`e}re, M.-L. Mugnier, and E.~Salvat.
\newblock On rules with existential variables: Walking the decidability line.
\newblock {\em Artificial Intelligence}, 175(9-10):1620--1654, 2011.

\bibitem[CGK08]{cali-gottlob-kifer:08}
A.~Cal\`{\i}, G.~Gottlob, and M.~Kifer.
\newblock Taming the infinite chase: Query answering under expressive
  relational constraints.
\newblock In {\em KR'08}, pages 70--80, 2008.

\bibitem[CGL{\etalchar{+}}07]{dl-lite:07}
D.~Calvanese, G.~De Giacomo, D.~Lembo, M.~Lenzerini, and R.~Rosati.
\newblock Tractable reasoning and efficient query answering in description
  logics: The {DL-Lite} family.
\newblock {\em J. Autom. Reasoning}, 39(3):385--429, 2007.

\bibitem[CGL09]{cali:09}
A.~Cal\`{\i}, G.~Gottlob, and T.~Lukasiewicz.
\newblock A general datalog-based framework for tractable query answering over
  ontologies.
\newblock In {\em PODS'09}, pages 77--86, 2009.

\bibitem[CGL12]{jws-12-cgl}
A.~Cal\`{\i}, G.~Gottlob, and T.~Lukasiewicz.
\newblock A general datalog-based framework for tractable query answering over
  ontologies.
\newblock {\em J. Web Sem.}, 14, 2012.

\bibitem[CGP10]{cgp:10}
A.~Cal\`{\i}, G.~Gottlob, and A.~Pieris.
\newblock Query answering under non-guarded rules in datalog+/-.
\newblock In {\em RR'10}, pages 1--17, 2010.

\bibitem[CR12]{d20-12-cr}
C.~Civili and R.~Rosati.
\newblock A broad class of first-order rewritable tuple-generating
  dependencies.
\newblock In {\em Datalog}, pages 68--80, 2012.

\bibitem[GOP11]{icde-11-gop}
G.~Gottlob, G.~Orsi, and A.~Pieris.
\newblock Ontological queries: Rewriting and optimization.
\newblock In {\em ICDE'11}, pages 2--13, 2011.

\bibitem[GS12]{kr-12-gs}
G.~Gottlob and T.~Schwentick.
\newblock Rewriting ontological queries into small nonrecursive datalog
  programs.
\newblock In {\em KR'12}, 2012.

\bibitem[ISG12]{isg:12}
Martha Imprialou, Giorgos Stoilos, and Bernardo~Cuenca Grau.
\newblock Benchmarking ontology-based query rewriting systems.
\newblock In {\em AAAI}, 2012.

\bibitem[KLMT12]{klmt:12}
M.~K{\"o}nig, M.~Lecl{\`e}re, M.-L. Mugnier, and M.~Thomazo.
\newblock A sound and complete backward chaining algorithm for existential
  rules.
\newblock In M.~Kr{\"o}tzsch and U.~Straccia, editors, {\em RR}, volume 7497 of
  {\em Lecture Notes in Computer Science}, pages 122--138. Springer, 2012.

\bibitem[KLMT13]{rr-13-klmt}
M.~K{\"o}nig, M.~Lecl{\`e}re, M.-L. Mugnier, and M.~Thomazo.
\newblock On the exploration of the query rewriting space with existential
  rules.
\newblock In {\em RR}, pages 123--137, 2013.

\bibitem[KLT{\etalchar{+}}11]{ijcai-11-kltwz}
R.~Kontchakov, C.~Lutz, D.~Toman, F.~Wolter, and M.~Zakharyaschev.
\newblock {T}he {C}ombined {A}pproach to {O}ntology-{B}ased {D}ata {A}ccess.
\newblock In {\em IJCAI}, pages 2656--2661, 2011.

\bibitem[KR11]{kr:11}
M.~Kr\"{o}tzsch and S.~Rudolph.
\newblock Extending decidable existential rules by joining acyclicity and
  guardedness.
\newblock In {\em IJCAI'11}, pages 963--968, 2011.

\bibitem[LTW09]{lutz:09}
C.~Lutz, D.~Toman, and F.~Wolter.
\newblock Conjunctive query answering in the description logic el using a
  relational database system.
\newblock In {\em IJCAI'09}, pages 2070--2075, 2009.

\bibitem[Mug11]{rr-11-m}
M.-L. Mugnier.
\newblock {Ontological Query Answering with Existential Rules}.
\newblock In {\em RR'11}, pages 2--23, 2011.

\bibitem[PUHM09]{eswc-09-phm}
H.~P\'erez-Urbina, I.~Horrocks, and B.~Motik.
\newblock Efficient query answering for owl 2.
\newblock In {\em ISWC'09}, pages 489--504, 2009.

\bibitem[RA10]{kr-10-ra}
R.~Rosati and A.~Almatelli.
\newblock Improving query answering over {DL-Lite} ontologies.
\newblock In {\em KR'10}, 2010.

\bibitem[RMC12]{kr-12-rc}
M.~Rodriguez-Muro and D.~Calvanese.
\newblock High performance query answering over {DL}-lite ontologies.
\newblock In {\em KR}, 2012.

\bibitem[SM96]{salvat-mugnier:96}
E.~Salvat and M.-L. Mugnier.
\newblock {S}ound and {C}omplete {F}orward and {B}ackward {C}hainings of
  {G}raph {R}ules.
\newblock In {\em ICCS'96}, volume 1115 of {\em LNAI}, pages 248--262.
  Springer, 1996.

\bibitem[Tho13]{ijcai-13-t}
M.~Thomazo.
\newblock Compact rewriting for existential rules.
\newblock In {\em IJCAI}, 2013.

\bibitem[VSS12]{dl-12-vss}
T.~Venetis, G.~Stoilos, and G.~B. Stamou.
\newblock Incremental query rewriting for {OWL} 2 {QL}.
\newblock In {\em Description Logics}, 2012.

\end{thebibliography}

\newpage
\section*{Appendix: Proof of Theorem \ref{CompleteWithSPU}}

To prove the completeness of the single-piece based operator, we first prove the following property:  

\begin{property}\label{propMGRifMGSinglePieceRew}
For any piece-unifier $\mu$ of $Q$ with $R$, there is a sequence of rewritings of $Q$ with
$R$  using exclusively most general single-piece unifiers and leading to a BCQ $Q^s$ such that
$Q^s \geq  \beta(Q,R, \mu)$.
\end{property}

\begin{proof}
We first introduce some notations. Given a partition $P$ and $x$ a term occurring in $P$, $P(x)$ is the class of $P$ that contains $x$. Let $P$ and $P'$ be two partitions such that the terms of $P'$ are included in the terms of $P$ and any class of $P'$ is included in a class of $P$: 
then we say that $P'$ is a subpart of $P$ (note that if $P'$ and $P$ are defined on the same set, it means that $P'$ is finer than $P$) 

Let $Pc_1,\ldots,Pc_n$ be the pieces of $Q'$ according to $\mu=(Q',H',P_u)$ and let $u$ be a substitution associated to $P_u$.  
Let $Q_0=Q, Q_1,\ldots Q_n=Q^s$ be a sequence of
rewritings of $Q$ built as follows: for $1\leq i\leq n$, $Q_i=\beta(Q_{i-1}, R_i,\mu_i)$
where $\mu_i=(Q'_i,H'_i,P_u^i)$ and $u_i$ is a substitution associated with $P_u^i$ with:
\begin{itemize}
\item $R_i$ is a safely renamed copy of $R$ by a variable renaming  $h_i$.
\item $H'_i$ is the image by $h_i$ of the subset of $H'$ unified by $u$ with $Pc_i$
\item $P_u^i$ is obtained from partition $h_i(P_u)$ (built from $P_u$ by applying $h_i$) by (1) restricting it to the terms of $Q'_i$ and $H'_i$  (2) refining it as much as possible while keeping the property that it is associated with a unifier of $H'_i$ and $Q'_i$. Note that $P_u^i$ is a subpart of $h_i(P_u)$.
\item Let $u^{\circ}_i = u_{i}  \circ u_{i-1} \circ \dots \circ u_1$. 
Let $P_u^{i \circ}$ be the partition assigned to $u^{\circ}_i$. We know that
$P_u^{i \circ}$ is the join of  $P_u^1, \dots P_u^{i}$, thus $P_u^{i \circ}$ is a subpart of $P_u^h$, the join of the $h_i(P_u)$ for $1\leq i\leq n$. Indeed, for each $i$, $P_u^{i}$ is a subpart of $h_i(P_u)$ and
the following property is easily checked: let $s_1$ and $s_2$ be substitutions with \emph{disjoint} domains, and $P_s^1$, $P_s^2$ be their associated partitions; then, the partition assigned to $s_1 \circ s_2$ (and to $s_2 \circ s_1$) is exactly the join of $P_s^1$ and $P_s^2$. 
\item $Q'_1 = Pc_1$ and for $i>1$, $Q'_i = u_{i-1}^{\circ}(Pc_i)$. We ensure  the property than $\forall i$, $u_{i-1}^{\circ}(Pc_i) \cap u_{i-1}^{\circ}(Q \setminus Q') = \emptyset$. If $u_{i-1}^{\circ}(Pc_i) \cap u_{i-1}^{\circ}(Q \setminus Q') \neq \emptyset$, we remove $\mu_i$ from the sequence because it is useless since $u_{i-1}^{\circ}(Pc_i) \subseteq u_{i-1}^{\circ}(Q \setminus Q')$. Indeed, let $a \in u_{i-1}^{\circ}(Pc_i) \cap u_{i-1}^{\circ}(Q \setminus Q')$, there are $b \in Pc_i$ and $b' \in Q \setminus Q'$, $b \neq b'$ such that $u_{i-1}^{\circ}(b) = u_{i-1}^{\circ}(b') = a$, so $\terms{b} \subseteq sep(Pc_i)$, so $\{b\}$ is a piece, so $Pc_i = \{b\}$ and then $u_{i-1}^{\circ}(Pc_i) = \{a\} \subseteq u_{i-1}^{\circ}(Q \setminus Q')$. For similar reasons, we ensure the property that $\forall i$, $\forall j > i$, $u_{i-1}^{\circ}(Pc_i) \cap u_{i-1}^{\circ}(Pc_j) = \emptyset$. 

\end{itemize}

We now show that:
\begin{enumerate}
	\item $\mu_i$ is a piece-unifier
	\item $\mu_i$ is a most general piece-unifier
	\item $\mu_i$ is a single-piece unifier
\end{enumerate}

For the first point:
\begin{itemize}
\item $Q'_i\subseteq Q_{i-1}$ since $\forall i$, $u_{i-1}^{\circ}(Pc_i) \cap u_{i-1}^{\circ}(Q \setminus Q') = \emptyset$ and $\forall i$, $\forall j > i$ $u_{i-1}^{\circ}(Pc_i) \cap u_{i-1}^{\circ}(Pc_j) = \emptyset$
\item $H'_i\subseteq \head{R_i}$ by construction.
\item $P_u^i$ satisfies the conditions of a piece-unifier because $P_u$ satisfies them and $P_u^i$ is a subpart of $h_i(P_u)$.
\end{itemize}

For the second point, since $P_u^i$ is the finest partition associated with a piece-unifier of $H'_i$ and $Q'_i$, we are sure that $\mu_i$ is a most general piece-unifier.

For the third point, note that each atom of $Q'_i$ corresponds to at least one atom of $Pc_i$. Thus if $Pc_i$ is composed of a unique atom, so is $H'_i$ which thus forms a single-piece.  Otherwise, $Pc_i$ is a single-piece from more than one atom; each atom $a$ of $Pc_i$ contains a variable $x$ such that $P_u(x)$ contains an existential variable $y$ which comes from the subset of $H'$ unified by $u$ with $Pc_i$.
Thus the corresponding atom $u_{i-1}^\circ(a)$ in $Q'_i$ is such that $P_u^i(u_{i-1}^\circ(x))$ contains the existential variable $h_i(y)$. So $Q'_i$ forms a single piece.

At the end of the sequence, $Q_n \subseteq u_n^{\circ}(Q \setminus Q')\cup \bigcup_{j \in 1..n}(u_n( \dots u_j(\bod{R_j})))$ and the terms of $P^{n \circ}_u$ are the same as the terms of
$P_u^h$.
Since $P^{n \circ}_u$ is a subpart of $P_u^h$, we can say that $P^{n \circ}_u$ is finer than $P_u^h$ so, there is a substitution $s$ such that $u^h = s \circ u^{\circ}_n$ and $s(u_n^{\circ}(Q \setminus Q')) = u^h(Q \setminus Q')$.
Let $h$ be the substitution obtained by making the union of the inverses of the $h_i$, then $h(u^h(Q \setminus Q')= u(Q \setminus Q')$, so $h\circ s$ is a homomorphism from $u_n^{\circ}(Q \setminus Q')$ to $u(Q \setminus Q')$.
Then we can prove that for all $j$, $1\leq j \leq n$, $h(s(u_n( \dots u_j(\bod{R_j})))) = u(\bod{R})$. 
Indeed, $u_n( \dots u_j(\bod{R_j})) = u_n( \dots u_1(\bod{R_j}))$ since the terms of $\bod{R_j}$ do not appear in $u_i$ ($i < j$).

To conclude the proof, we have $h(s(Q_n)) \subseteq  u(\bod{R}) \cup u(Q \setminus Q') = \beta(Q,\mu, R)$, hence $h\circ s$ is a homomorphism from $Q_n$ to $\beta(Q,\mu, R)$, thus $Q_n \geq \beta(Q,\mu, R)$.
\end{proof}

\paragraph{Theorem \ref{thm-sgp}}
\emph{Given a BCQ $Q$ and a set of rules $\mathcal R$, the set of rewritings of $Q$ obtained by
considering exclusively most general single-piece unifiers is sound and complete. }

\medskip
\begin{proof}
Soundness holds trivially since a single-piece unifier is a piece-unifier. \\
For completeness, thanks to Theorem \ref{theo:sc}, we just have to show by induction on $k$, the length of the rewriting sequence
leading from $Q$ to a $k$-piece-rewriting of $Q$, that: for any
$k$-piece-rewriting $Q^r$ of $Q$, there exists $Q^{s}$ a piece-rewriting of
$Q$ obtained by using exclusively most general single-piece unifiers such that $Q^{s}\geq
Q^r$.\\
For $k=0$ the property is trivially satisfied.\\
For $k\geq 1$, one has $Q^r=\beta(Q^{r'},R,\mu)$, with $Q^{r'}$ being a
piece-rewriting of $Q$ obtained by a piece-rewriting sequence of length $k-1$. By
induction hypothesis, there exists $Q^{s'}$ a piece-rewriting of $Q$ obtained by
using exclusively single-piece unifiers such that $Q^{s'}\geq Q^{r'}$. By Lemma \ref{propCompRewifCompQ}, either $Q^{s'}\geq Q^r$, or there is a piece-unifier $\mu'$ of
$Q^{s'}$ with $R$ such that  $\beta(Q^{s'},R,\mu')\geq Q^r$. In this latter case, thanks
to Property \ref{propMGRifMGSinglePieceRew}, there is a sequence of rewritings of $Q^{s'}$
with $R$ using only single-piece unifiers and leading to a CQ $Q^{s}$ such that
$Q^{s}\geq \beta(Q^{s'},R,\mu')$.
\end{proof}

\end{document}